%% file: main.tex
\documentclass{article}
\usepackage{arxiv}
\usepackage[T1]{fontenc}
\usepackage{microtype}
\usepackage{amsmath,amssymb,amsthm,mathtools,bm}
\usepackage{booktabs,multirow,adjustbox,float,caption}
\usepackage{xcolor}
\usepackage[round,authoryear]{natbib}
\usepackage[colorlinks=true,allcolors=blue]{hyperref}
\hypersetup{
  pdftitle={Certifying Whether to Adapt: Baseline-Relative PAC-Bayes--MMD Guarantees for Test-Time Adaptation},
  pdfauthor={Ahanaf Hasan Ariq; Hasaan Ahmad; Aayush Parekh; Abdul Monaf Chowdhury}
}
\usepackage{pgfplots}
\pgfplotsset{compat=1.18}
\usepgfplotslibrary{groupplots}

\newtheorem{theorem}{Theorem}
\newtheorem{proposition}{Proposition}
\newtheorem{lemma}{Lemma}
\newtheorem{corollary}{Corollary}

\newtheorem{assumption}{Assumption}

\DeclareMathOperator{\MMD}{MMD}
\DeclareMathOperator{\KL}{KL}
\DeclareMathOperator{\TV}{TV}
\DeclareMathOperator{\kl}{kl}

\newcommand{\E}{\mathbb E}
\newcommand{\R}{\mathbb R}
\newcommand{\cH}{\mathcal H}
\newcommand{\ind}{\mathbf 1}

\title{Certifying Whether to Adapt:\\
Baseline-Relative PAC-Bayes--MMD Guarantees for Test-Time Adaptation}
\author{
Ahanaf Hasan Ariq\\{\normalfont\small Daffodil International School}
\And
Hasaan Ahmad\\{\normalfont\small University of the West of England, Bristol}
\AND
Aayush Parekh\\{\normalfont\small University of Nottingham}
\And
Abdul Monaf Chowdhury\\{\normalfont\small University of Dhaka}
}
\date{}
\renewcommand{\shorttitle}{Certifying Whether to Adapt}

\definecolor{cblue}{HTML}{2A78D6}
\definecolor{corange}{HTML}{EB6834}
\definecolor{caqua}{HTML}{1BAF7A}
\definecolor{cyellow}{HTML}{EDA100}
\definecolor{cmagenta}{HTML}{E87BA4}
\definecolor{cgreen}{HTML}{008300}
\definecolor{cviolet}{HTML}{4A3AA7}
\definecolor{cred}{HTML}{E34948}
\definecolor{cgray}{HTML}{7A7974}
\pgfplotsset{
 every axis/.append style={font=\scriptsize,grid=major,grid style={gray!20},
 axis lines*=left,tick align=outside,title style={font=\scriptsize,align=center},
 legend cell align=left,legend style={font=\scriptsize,draw=none,fill=none}},
 figure legend/.style={draw=gray!35,fill=white,fill opacity=.96,text opacity=1,
 rounded corners=1pt,inner sep=2pt,cells={anchor=west}},
 eb/.style={error bars/.cd,y dir=both,y explicit,error bar style={line width=.5pt},
 error mark options={rotate=90,mark size=1.5pt}},
 m_cert_directional/.style={color=cblue,mark=*,line width=.9pt,mark size=1.6pt},
 m_cert_radial/.style={color=corange,mark=square*,line width=.9pt,mark size=1.5pt},
 m_cert_joint/.style={color=caqua,mark=diamond*,line width=.9pt,mark size=1.8pt,densely dashed},
 m_direct_target_audit/.style={color=cyellow,mark=triangle*,line width=.9pt,mark size=1.8pt},
 m_direct_conditional_model/.style={color=cmagenta,mark=triangle*,mark options={rotate=180},line width=.9pt,mark size=1.8pt},
 m_always_adapt/.style={color=cred,mark=x,line width=.9pt,mark size=2pt},
 m_never_adapt/.style={color=cgray,mark=+,line width=.9pt,mark size=2pt},
 m_absolute_directional/.style={color=cviolet,mark=pentagon*,line width=.9pt,mark size=1.7pt},
 m_absolute_radial/.style={color=cgreen,mark=otimes*,line width=.9pt,mark size=1.7pt}
}

\begin{document}
\maketitle

\begin{abstract}
Test-time adaptation (TTA) needs a decision about an update, not only an absolute-risk estimate. We give PAC-Bayesian certificates for the target excess risk of a candidate relative to a pilot-fixed baseline. A radial certificate uses an MMD ambiguity radius; a direction-sensitive certificate retains the sign of the empirical RKHS shift and can therefore certify target improvement even when source risk increases. For finite binary state spaces we derive statewise marginal and signed conditional-shift refinements, with penalties localized to candidate--baseline disagreements. A source-worsening update with target excess risk $-0.576$ is accepted in all 1,000 repetitions by the directional certificate and never by the radial certificate; learned finite-state updates with excess risks $-0.800$ and $-0.401$ are likewise accepted in every repetition. On Covertype, the certificate safely rejects all updates and blocks geographic-shift adaptations whose unconditional deployment increases risk by up to $2.75$ percentage points in prespecified base libraries; an exploratory stress test reaches $5.66$ points. A feature-quantized CIFAR-10-C and ImageNet-C study is structurally vacuous and is reported only as a diagnostic limitation, not as a practical neural certificate.
\end{abstract}

\section{Introduction}
TTA methods such as TENT, EATA, SAR, and MEMO update a source model using deployment inputs \citep{wang2021tent,niu2022eata,niu2023sar,zhang2022memo}. Such updates may help or harm. Existing PAC-Bayesian adaptation bounds control absolute target risk using discrepancies, disagreement, or joint error \citep{germain2013pac,germain2016domain,breitholtz2022practicality}; kernel distributionally robust optimization gives related MMD ambiguity sets \citep{staib2019mmd,zhu2021kernel}. We instead certify the paired quantity
\begin{equation}
 D_P(\rho,h_0)=R_P(\rho)-R_P(h_0),
 \label{eq:contrast}
\end{equation}
where $h_0$ is fixed before certification. Pairing preserves cancellation:
\begin{equation}
 \sup_Q\{R_Q(\rho)-R_Q(h_0)\}
 \leq \sup_QR_Q(\rho)-\inf_QR_Q(h_0).
 \label{eq:paired-robust}
\end{equation}

Our contributions are: (i) radial and direction-sensitive PAC-Bayes--MMD certificates for~\eqref{eq:contrast}; (ii) a finite-state sharpening that matches the statewise implementation and a signed conditional audit that can exploit beneficial label shift; (iii) disagreement-localized witnesses valid under posterior selection; and (iv) an empirical study separating complete finite-state certificates, conservative real-data deployment, and deliberately non-certified neural diagnostics. We do not claim the first PAC-Bayes--MMD bound or a universal safe-improvement principle.

\section{Setting and Credal Interpretation}
\label{sec:setting}
Let $P_s,P_t$ be laws on $\mathcal X\times\mathcal Y$, let $\ell\in[0,1]$, and define $R_P(\rho)=\E_{h\sim\rho,(X,Y)\sim P}\ell(h,X,Y)$. An independent pilot fixes $h_0$, the prior $\pi$, hypothesis class, representation, finite state map when used, and finite kernel grid. The labeled source certificate sample has size $n$; source and target marginal samples have sizes $m,r$. Additional audits are independent as specified in the appendix. The posterior may use the source sample and unlabeled target data because the PAC-Bayes event is simultaneous in $\rho$.

Fix a source conditional $p_s(dy\mid x)$ on the full admissible domain and define
\begin{align}
 f_\rho^\Delta(x)&=\E_{\rho,p_s(\cdot\mid x)}[\ell(h,x,Y)-\ell(h_0,x,Y)],\\
 \nu_\rho&\geq \E_{P_t^X}|f_{t,\rho}^\Delta(X)-f_\rho^\Delta(X)|.
 \label{eq:relative-conditional}
\end{align}
Unlabeled inputs cannot identify $\nu_\rho$ distribution-free; it must be assumed or audited. For a kernel $k$ with RKHS $\cH$, $k(x,x)\leq\kappa^2$, let $\MMD_k(P,Q)=\|\mu_P-\mu_Q\|_\cH$.

For the fixed source conditional, the joint credal set
\begin{equation}
 \mathcal C_{\varepsilon,\eta}=\left\{Q:
 \begin{aligned}
 \MMD_k(P_s^X,Q^X)&\leq\varepsilon,\\
 \E_{Q^X}\TV(Q(\cdot\mid X),p_s(\cdot\mid X))&\leq\eta
 \end{aligned}\right\}
 \label{eq:credal}
\end{equation}
separates marginal ambiguity from label-mechanism ambiguity. Since the signed loss has range length two, the conditional contribution is at most $2\eta$.

\begin{proposition}[Unlabeled non-identifiability]
\label{prop:impossible}
No input-only, distribution-free procedure can give a uniformly informative upper confidence bound on the integrated conditional total variation in~\eqref{eq:credal}.
\end{proposition}
\begin{proof}
Fix an input marginal, let source labels be identically zero, and compare targets with the same inputs but labels identically zero or identically one. The conditional budgets are zero and one, while every unlabeled observation has the same distribution.
\end{proof}

\section{Baseline-Relative Certificates}
\begin{assumption}[Centered RKHS witness]
\label{ass:rkhs}
For each eligible $\rho$, there are $c_\rho^\Delta\in\R$ and $g_\rho^\Delta\in\cH$ such that
\begin{equation}
 \|g_\rho^\Delta\|_\cH\leq B_\rho^\Delta,\qquad
 \|f_\rho^\Delta-c_\rho^\Delta-g_\rho^\Delta\|_\infty\leq a_\rho^\Delta.
 \label{eq:contrast-assumption}
\end{equation}
\end{assumption}
Universality, a fitted KRR norm, or a small sample residual does not by itself supply these population witnesses.

\begin{lemma}[Deterministic paired transfer]
\label{lem:transfer}
Under Assumption~\ref{ass:rkhs},
\begin{align}
 D_{P_t}(\rho,h_0)\leq{}&D_{P_s}(\rho,h_0)
 +B_\rho^\Delta\MMD_k(P_s^X,P_t^X)\notag\\
 &+2a_\rho^\Delta+\nu_\rho.
 \label{eq:transfer}
\end{align}
\end{lemma}

For $q\in[0,1]$, let $p_+(q,C)=\sup\{p:\kl(q\|p)\leq C\}$ and let $p_-(q,C)$ be the lower endpoint. On the source sample set $z_h=(1+\ell(h)-\ell(h_0))/2$, let $q_\rho^\Delta$ be its empirical Gibbs mean, and define
\begin{align}
 C_\rho^\Delta&=\frac{\KL(\rho\|\pi)+\log(2\sqrt n/\delta_\Delta)}n,\notag\\
 S_\rho^+&=2p_+(q_\rho^\Delta,C_\rho^\Delta)-1.
 \label{eq:source}
\end{align}
PAC-Bayes-kl gives
\begin{equation}
 D_{P_s}(\rho,h_0)\leq S_\rho^+\quad\text{simultaneously for all }\rho
 \label{eq:pb}
\end{equation}
with probability $1-\delta_\Delta$ \citep{maurer2004note,alquier2024user}. For $d_h=\ell(h)-\ell(h_0)$, separately controlling $(d_h)_+$ and $(-d_h)_+$ gives the alternative
\begin{equation}
 D_{P_s}(\rho,h_0)\leq p_+(q_\rho^+,C_\rho^+)-p_-(q_\rho^-,C_\rho^-),
 \label{eq:parts}
\end{equation}
provided both events are charged.

Let $\widehat d_j=\|\widehat\mu_s-\widehat\mu_t\|$ and
\begin{equation}
 b_{k_j}=\kappa_j(1+\sqrt{2\log(2J/\delta_m)})(m^{-1/2}+r^{-1/2}),
 \label{eq:radius}
\end{equation}
so $d_j^+=\min\{2\kappa_j,\widehat d_j+b_{k_j}\}$ simultaneously for $J$ pilot-fixed kernels.

\begin{theorem}[Radial paired certificate]
\label{thm:relative}
With probability at least $1-\delta_\Delta-\delta_m$, simultaneously for every eligible $(\rho,j)$,
\begin{equation}
 D_{P_t}(\rho,h_0)\leq V_{\rho,j}^{\rm rad}:=
 S_\rho^++B_{\rho,j}^\Delta d_j^++2a_{\rho,j}^\Delta+\nu_\rho.
 \label{eq:relative-bound}
\end{equation}
Consequently, for $Q\in\mathcal C_{\varepsilon,\eta}$,
\begin{equation}
 \sup_QD_Q(\rho,h_0)\leq S_\rho^++B_\rho^\Delta\varepsilon+2a_\rho^\Delta+2\eta.
 \label{eq:robust-paired}
\end{equation}
\end{theorem}

\begin{proposition}[Direction-sensitive transfer]
\label{prop:directional}
If $g_{\rho,j}^\Delta$ is evaluable, then on the same event
\begin{align}
 V_{\rho,j}^{\rm dir}={}&S_\rho^+
 +\widehat\E_t g_{\rho,j}^\Delta-\widehat\E_s g_{\rho,j}^\Delta\notag\\
 &+B_{\rho,j}^\Delta b_{k_j}+2a_{\rho,j}^\Delta+\nu_\rho
 \label{eq:directional}
\end{align}
upper-bounds $D_{P_t}(\rho,h_0)$. Its minimum with $V^{\rm rad}$ is valid without an extra event.
\end{proposition}

For a finite full domain with positive definite Gram matrix $K$ and approximate conditional-contrast vector $v_\rho$, the norm-minimizing centered witness is
\begin{equation}
 c_\rho=\frac{\bm1^\top K^{-1}v_\rho}{\bm1^\top K^{-1}\bm1},\qquad
 B_\rho^2=(v_\rho-c_\rho\bm1)^\top K^{-1}(v_\rho-c_\rho\bm1).
 \label{eq:finite-norm}
\end{equation}
\begin{proposition}[Finite-domain centered certificate]
\label{prop:finite-domain}
If $|f_\rho^\Delta(x_i)-v_{\rho,i}|\leq a_\rho^\Delta$ on every admissible state, then~\eqref{eq:finite-norm} and residual $a_\rho^\Delta$ satisfy Assumption~\ref{ass:rkhs}.
\end{proposition}

\section{Finite-State Sharpening}
Let $\mathcal X=\{1,\ldots,d\}$, $q_{D,i}=P_D(X=i)$, $\eta_{D,i}=P_D(Y=1\mid X=i)$, $\alpha_{\rho,i}=\Pr_\rho(h(i)=1)-h_0(i)$, and $w_{\rho,i}=(1-2\eta_{s,i})\alpha_{\rho,i}$. The exact decomposition is
\begin{equation}
 D_{P_t}=D_{P_s}+\langle q_t-q_s,w_\rho\rangle
 +\sum_i2q_{t,i}\alpha_{\rho,i}(\eta_{s,i}-\eta_{t,i}).
 \label{eq:finite-decomposition}
\end{equation}
Suppose simultaneous intervals give $|q_{D,i}-\widehat q_{D,i}|\leq h_{D,i}$ and $|w_{\rho,i}-\widehat w_{\rho,i}|\leq r_{\rho,i}$ for every state and posterior.

\begin{proposition}[Statewise marginal transfer]
\label{prop:statewise}
On this event, simultaneously for all $\rho$,
\begin{align}
 T_\rho^{\rm st}={}&\langle\widehat q_t-\widehat q_s,\widehat w_\rho\rangle
 +\sum_i(h_{t,i}+h_{s,i})|\widehat w_{\rho,i}|\notag\\
 &+\sum_i(|\widehat q_{t,i}-\widehat q_{s,i}|+h_{t,i}+h_{s,i})r_{\rho,i}
 \label{eq:statewise}
\end{align}
upper-bounds $\langle q_t-q_s,w_\rho\rangle$.
\end{proposition}

Let $|\eta_{D,i}-\widehat\eta_{D,i}|\leq e_{D,i}$,
$\widehat c_{\rho,i}=2\alpha_{\rho,i}
(\widehat\eta_{s,i}-\widehat\eta_{t,i})$,
$r_{\rho,i}^c=2|\alpha_{\rho,i}|(e_{s,i}+e_{t,i})$, and
$\overline q_{t,i}=\min\{1,\widehat q_{t,i}+h_{t,i}\}$. Define
\begin{align}
 C_\rho^{\rm st}={}&\langle\widehat q_t,\widehat c_\rho\rangle
 +\langle h_t,|\widehat c_\rho|\rangle\notag\\
 &+\langle\overline q_t,r_\rho^c\rangle.
 \label{eq:signed-conditional}
\end{align}
\begin{proposition}[Sign-aware conditional audit]
\label{prop:signed-conditional}
On the interval event, the final term of~\eqref{eq:finite-decomposition} is at most $C_\rho^{\rm st}$ simultaneously for all posteriors, including audit-selected posteriors.
\end{proposition}
For the delta kernel, the implementation uses the generic radial arm
$T_\rho^{\rm rad}=d^+(\|\widehat w_\rho\|_2+\|r_\rho\|_2)$ and the
statewise directional arm $T_\rho^{\rm st}$. The joint certificate is
$S_\rho^++\min\{T_\rho^{\rm rad},T_\rho^{\rm st}\}
+C_\rho^{\rm st}$. Thus the statewise concentration is proved rather
than a generic embedding radius relabeled. Proofs and interval
construction are in Appendix~\ref{app:statewise}.

\begin{corollary}[Accept or retain]
\label{cor:gate}
Select any eligible posterior and valid certificate covered by the simultaneous events. Deploy it only if its upper bound is at most $-\xi$; otherwise retain $h_0$. On the joint event, deployment never increases target risk, and acceptance improves it by at least $\xi$.
\end{corollary}

\begin{proposition}[Analytical negative certificate]
\label{prop:example}
On $\{0,1\}$ with the delta kernel, $P_s(X=1)=0.2$, $P_t(X=1)=t\in[0.18,0.22]$, and known channel $P(Y\ne X\mid X)=0.02$, take $h_0(x)=0$, $h_1(x)=x$, a uniform prior, and $\rho=\delta_{h_1}$. At $n=m=r=10^4$ and $\delta_\Delta=\delta_m=0.025$,
\begin{equation}
 q_\rho^\Delta\leq0.42,\quad \widehat d_k\leq0.06
 \quad\Longrightarrow\quad V_{\rho,1}^{\rm rad}<-0.021.
 \label{eq:example-conditions}
\end{equation}
For each fixed $t$ in the interval, these sufficient conditions hold with probability greater than $0.98$.
\end{proposition}

\section{Experiments}
\label{sec:experiments}
We evaluate exact two-state constructions, learned threshold libraries on $d\in\{8,32\}$ states, conditional-shift stress tests, Covertype \citep{blackard1998covertype}, and a diagnostic TENT panel on CIFAR-10-C/ImageNet-C \citep{hendrycks2019benchmarking}. Complete-certificate experiments use pilot-fixed libraries, exact evaluation only after the decision, total failure probability $0.05$, and independent repetitions. Central finite-state runs use $n=m=r=10^4$; Experiment 1 has 1,000 repetitions, Experiments 2--3 have 100, and Covertype has five pilot splits with ten repetitions each. Comparators include never/always adapt, direct conditional-model bounds, and direct labeled-target audits at matched information budgets. Full protocols, confidence ledgers, ablations, and all figures are in Appendix~\ref{app:experiments}.

\paragraph{Directional certificates accept useful updates.}
In the two-state source-worsening case, the candidate has $D_s=+0.576$ and $D_t=-0.576$. Over 1,000 repetitions, the radial certificate never accepts, whereas the directional certificate accepts every repetition at margin $0.30$; harmful and neutral controls are never accepted. In learned $d=8$ libraries, the directional gate accepts every concentrated-shift update and attains target excess risks $-0.800$ and $-0.401$, while the radial gate never accepts. It correctly rejects a moderate update with available improvement $-0.098$ when its mean bound is $+0.038$.

\begin{table*}[t]
\caption{Core complete-certificate results. Acceptance is the percentage of independent repetitions; deployed excess includes baseline retention. Exact 95\% binomial intervals for 0/1000 and 1000/1000 are $[0,0.37]\%$ and $[99.63,100]\%$.}
\label{tab:core}
\centering\small
\begin{tabular}{lrrrrr}
\toprule
Setting & Repetitions & True $D_t$ & Radial acc. & Directional acc. & Deployed $D_t$\\
\midrule
Two-state source-worsening & 1,000 & $-0.576$ & 0 & 100 & $-0.576$\\
Two-state harmful control & 1,000 & $+0.576$ & 0 & 0 & $0$\\
$d=8$, concentrated $\lambda=.9$ & 100 & $-0.800$ & 0 & 100 & $-0.800$\\
$d=8$, concentrated $\lambda=.5$ & 100 & $-0.401$ & 0 & 100 & $-0.401$\\
$d=8$, moderate $\beta=2$ & 100 & $-0.098$ & 0 & 0 & $0$\\
\bottomrule
\end{tabular}
\end{table*}

\begin{figure*}[t]
\centering
\begin{tikzpicture}
\begin{groupplot}[group style={group size=3 by 1,horizontal sep=1.25cm},width=.285\textwidth,height=4.4cm,grid=none,ymajorgrids]
\nextgroupplot[ylabel={bound / true excess},xtick={0,1},xticklabels={radial,directional},title={$d=8$, $\lambda=.9$},ymin=-1,ymax=1.4,
 legend style={figure legend,font=\tiny,at={(0.98,0.98)},anchor=north east,legend columns=1}]
\addplot[ybar,fill=corange,draw=none,forget plot] coordinates {(0,1.245) (1,-.663)};
\addlegendimage{area legend,fill=corange,draw=none}\addlegendentry{certificate}
\addplot[black,dashed,mark=none] coordinates {(-.4,-.800) (1.4,-.800)};\addlegendentry{true excess}
\nextgroupplot[ylabel={deployed excess (pp)},xtick={0,1,2,3},xticklabels={{base 1--3},{base 1--4},{rich 1--3},{rich 1--4}},x tick label style={rotate=25,anchor=east},title={Covertype},ymin=-.5,ymax=6.2,
 legend style={figure legend,font=\tiny,at={(0.02,0.98)},anchor=north west,legend columns=1}]
\addplot[ybar,bar width=5pt,fill=cred,draw=none,forget plot] coordinates {(0,2.75) (1,.89) (2,2.88) (3,5.66)};
\addlegendimage{area legend,fill=cred,draw=none}\addlegendentry{always adapt}
\addplot[m_cert_directional,only marks,mark size=1.8pt] coordinates {(0,0) (1,0) (2,0) (3,0)};\addlegendentry{certificate}
\nextgroupplot[title={rare disagreement audit},xmode=log,xlabel={$P_t(G)$},ylabel={mean half-width},ymin=0,ymax=.65,
 legend style={figure legend,font=\tiny,at={(0.98,0.98)},anchor=north east,legend columns=1}]
\addplot+[m_cert_radial] coordinates {(.01,.581) (.05,.211) (.2,.085)};\addlegendentry{uniform}
\addplot+[m_cert_directional] coordinates {(.01,.032) (.05,.033) (.2,.032)};\addlegendentry{targeted}
\end{groupplot}
\end{tikzpicture}
\caption{Left: preserving direction turns a vacuous radial bound into a negative certificate (dashed: true excess). Center: unconditional Covertype adaptation increases risk; certification rejects and retains the baseline; ``rich'' is exploratory. Right: targeted auditing sharply reduces uncertainty for rare disagreement regions.}
\label{fig:headline}
\end{figure*}

\paragraph{Localization identifies where labels matter.}
With equal input marginals, MMD is exactly zero while conditional shift changes the sign of $D_t$. When shift occurs outside the disagreement set, the exact paired conditional contribution is zero; the localized audited term stays near $0.035$ while the global term reaches $1.18$. For a rare disagreement region with target mass $0.01$, targeted auditing reduces the mean statewise half-width from $0.581$ to $0.032$, while screening roughly $2\times10^5$ unlabeled inputs for 2,000 labels.

\paragraph{Covertype demonstrates risk aversion.}
We quantize cover types 1 versus 2 into 8 or 32 pilot-fixed states and audit controlled reweightings and geographic shifts. Here the certificate does exactly what an accept-or-retain mechanism is designed to do: it blocks every update when the evidence is insufficient. In the prespecified base libraries, unmonitored geographic-shift adaptation is harmful in up to 100\% of repetitions and increases mean risk by up to $2.75$ percentage points, whereas the gate retains the baseline and incurs zero excess. An exploratory richer-library stress test increases this gap to $5.66$ points. A direct audit with 10,000 target labels captures some sub-percentage gains, exposing the power cost of risk aversion and the price of transfer relative to direct supervision.

\paragraph{Vision is diagnostic, not a neural certificate.}
A TENT shadow update is evaluated on three CIFAR-10-C corruptions and ImageNet-C severity 5. Quantizing penultimate features into 32 states does not make the unrestricted adapted network state-constant, and the declared within-state invariant channel is violated by corruption. The ImageNet run also uses brightness severity 1 as a labeled source proxy because clean validation data were unavailable. Therefore this panel is not an instantiation of our theorems. The putative bound rejects every run and is structurally vacuous: its mean residual-transfer term is $45.9$ percentage points. We report it to delimit scope, not as a practical deployment certificate.

\section{Related Work and Limitations}
PAC-Bayes--MMD transfer and its practical vacuity are established \citep{breitholtz2022practicality}; PAC-Bayesian domain adaptation and IPM-based PAC-Bayes analyses also predate this work \citep{germain2013pac,germain2016domain,amit2022ipm}. Split-kl and example-difficulty bounds provide sharper excess-loss control \citep{wu2022split,biggs2023tighter}; robust baseline regret is also established \citep{petrik2016safe}. Recent confidence-sequence monitoring for TTA raises alarms when predictive risk crosses a threshold \citep{schirmer2025monitoring}; our static question is instead whether a particular candidate is certified to improve on its baseline. Broad TTA benchmarking further shows strong protocol and backbone sensitivity \citep{yu2023benchmark}. Our contribution is the paired, direction-sensitive composition and its auditable finite-state realization. Direct labeled-target auditing can be tighter and is always reported at matched information.

The guarantees require a fixed representation, valid population loss witnesses, and a conditional-shift assumption or audit. Covertype shows that safe rejection may coexist with low power. The vision diagnostic shows that feature quantization does not certify an unrestricted deep TTA model. Static IID bounds also do not certify arbitrary continual streams. These limitations are structural rather than presentation choices.

\section{Conclusion}
Baseline-relative certification converts shift control into an accept-or-retain decision. Direction-sensitive transfer certifies updates that radial ambiguity bounds necessarily reject, while localized audits identify where conditional labels matter. The same machinery can rationally abstain on real shifts; complete non-vacuous certification of unrestricted neural TTA remains open.

\clearpage
\appendix
\section*{Appendices}

\input{theory_appendix}

\section{Statewise Finite-State Proofs}
\label{app:statewise}
\begin{proof}[Proof of Proposition~\ref{prop:statewise}]
Write $q_t-q_s=(\widehat q_t-\widehat q_s)+u$ and $w_\rho=\widehat w_\rho+e$, where $|u_i|\leq h_{t,i}+h_{s,i}$ and $|e_i|\leq r_{\rho,i}$. Expanding the inner product gives the empirical term plus $\langle u,\widehat w_\rho\rangle+\langle\widehat q_t-\widehat q_s,e\rangle+\langle u,e\rangle$. Bounding each coordinate by its absolute envelope proves~\eqref{eq:statewise}. The coordinate event is posterior-independent, so the result is simultaneous for every audit-selected posterior.
\end{proof}
\begin{proof}[Proof of Proposition~\ref{prop:signed-conditional}]
Write $q_{t,i}=\widehat q_{t,i}+u_i$ and
$c_{\rho,i}=\widehat c_{\rho,i}+v_i$, where
$|u_i|\leq h_{t,i}$, $|v_i|\leq r_{\rho,i}^c$, and
$0\leq q_{t,i}\leq\overline q_{t,i}$. Then
$q_{t,i}c_{\rho,i}\leq
\widehat q_{t,i}\widehat c_{\rho,i}
+h_{t,i}|\widehat c_{\rho,i}|
+\overline q_{t,i}r_{\rho,i}^c$.
Summing proves~\eqref{eq:signed-conditional}. The interval event is
posterior-independent, hence the assertion is simultaneous in $\rho$.
\end{proof}
For the implementation, multinomial coordinates are treated as Bernoulli indicators and simultaneous Bernoulli-kl intervals are inverted for every state. Conditional-label intervals are constructed given the observed state counts. For binary zero--one loss,
$\widehat w_{\rho,i}=(1-2\widehat\eta_{s,i})\alpha_{\rho,i}$ and
$r_{\rho,i}=2e_{s,i}|\alpha_{\rho,i}|$. The conditional center and
radius are those preceding Proposition~\ref{prop:signed-conditional}.
Known channels collapse the corresponding intervals to points. This
is precisely the statewise concentration used in the reported
finite-state experiments.

\clearpage
\section{Complete Experimental Results}
\label{app:experiments}
\input{experiment_appendix}

\bibliographystyle{plainnat}
\bibliography{references}

\end{document}

%% file: theory_appendix.tex
\section{Supporting Absolute-Risk and Deterministic Results}
\label{app:support}
For completeness, define the source conditional loss
$f_\rho(x)=\E_{h\sim\rho,p_s(\cdot\mid x)}\ell(h,x,Y)$,
its target counterpart $f_{t,\rho}$, and
$\Delta_\rho(P_t)=\E_{P_t^X}|f_{t,\rho}-f_\rho|$. Suppose
$f_\rho=c_\rho+g_\rho+e_\rho$ with
$\|g_\rho\|_\cH\leq B_\rho$ and $\|e_\rho\|_\infty\leq a_\rho$.
On the labeled source sample let
\begin{equation}
 q_\rho=\frac1n\sum_{i=1}^n\E_\rho\ell(h,X_i,Y_i),\qquad
 C_\rho=\frac{\KL(\rho\|\pi)+\log(2\sqrt n/\delta_s)}n,
 \label{eq:absolute-source}
\end{equation}
and let $p_-(q,C)$ be the lower Bernoulli-kl inverse. The standard
PAC-Bayes-kl event is
\begin{equation}
 \kl(q_\rho\|R_{P_s}(\rho))\leq C_\rho
 \quad\text{simultaneously for every }\rho.
 \label{eq:absolute-pb}
\end{equation}
\begin{theorem}[Population PAC-Bayes--MMD envelope]
\label{thm:population}
On event~\eqref{eq:absolute-pb}, every eligible posterior and target law satisfy
\begin{equation}
 R_{P_t}(\rho)\in[p_-(q_\rho,C_\rho)-\tau_\rho,
 p_+(q_\rho,C_\rho)+\tau_\rho]\cap[0,1],
 \label{eq:population}
\end{equation}
where $\tau_\rho=B_\rho\MMD_k(P_s^X,P_t^X)+2a_\rho+\Delta_\rho(P_t)$.
\end{theorem}
\begin{proof}
Decompose $R_{P_t}(\rho)-R_{P_s}(\rho)$ into the target conditional remainder and the marginal expectation difference of $f_\rho=c_\rho+g_\rho+e_\rho$. The intercept cancels, Cauchy--Schwarz bounds the RKHS term, and the two residual expectations contribute at most $2a_\rho$. Apply this deterministic bound to both endpoints of event~\eqref{eq:absolute-pb}; no additional probability is spent.
\end{proof}

\begin{proposition}[Population width and confidence envelope]
\label{prop:credal}
Let $t_\rho=B_\rho\varepsilon+2a_\rho+\eta$. Then
\begin{align}
 R_{P_s}(\rho)-t_\rho&\leq\underline R\leq\overline R\leq R_{P_s}(\rho)+t_\rho,\\
 \overline R-\underline R&\leq\min\{1,2t_\rho\}.
 \label{eq:width}
\end{align}
On~\eqref{eq:absolute-pb}, $[p_-(q_\rho,C_\rho)-t_\rho,p_+(q_\rho,C_\rho)+t_\rho]\cap[0,1]$ contains the entire robust risk interval simultaneously for all $\varepsilon,\eta\geq0$.
\end{proposition}
\begin{proof}
Every member of the credal set obeys the same deterministic transfer bound with remainder at most $\eta$. Take the infimum and supremum, then apply the one source event. No union bound over target laws is needed.
\end{proof}
The population width is not the width of the confidence envelope and need not attain this upper bound. A certified upper risk below a prespecified tolerance supports deployment within the stated set; it does not compare actual risks of two predictors.

\subsection{Deterministic prediction}
\label{sec:det}
Use zero--one loss and a finite label space. Let $G_\rho$ denote Gibbs prediction and $V_\rho(x)$ the label with largest posterior vote mass, with fixed tie breaking. For a representative $h_{\rm rep}$, define $D_t(h_{\rm rep},G_\rho)=\E_{X\sim P_t^X,h\sim\rho}\ind\{h_{\rm rep}(X)\ne h(X)\}$.
\begin{proposition}[Standard deterministic bridges]
\label{prop:det}
For every target law,
\begin{align}
 R_t(V_\rho)&\leq2R_t(G_\rho),\label{eq:vote}\\
 R_t(h_{\rm rep})&\leq R_t(G_\rho)+D_t(h_{\rm rep},G_\rho).
 \label{eq:direct}
\end{align}
Fix $h_{\rm rep},\rho$ independently of i.i.d. target audit inputs $X_1,\ldots,X_u$ and independent draws $H_i\sim\rho$. Let $\widehat D=u^{-1}\sum_i\ind\{h_{\rm rep}(X_i)\ne H_i(X_i)\}$. If $U_\rho$ is a valid Gibbs upper certificate, then, with additional failure probability $\delta_D$,
\begin{equation}
 R_t(h_{\rm rep})\leq\min\{1,U_\rho+\widehat D+
                    \sqrt{\log(1/\delta_D)/(2u)}\}.
 \label{eq:rep}
\end{equation}
\end{proposition}
\begin{proof}
When the vote is incorrect, the correct label has posterior mass at most $1/2$, proving~\eqref{eq:vote}. For every $h,x,y$, $\ind\{h_{\rm rep}(x)\ne y\}\leq\ind\{h(x)\ne y\}+\ind\{h_{\rm rep}(x)\ne h(x)\}$, giving~\eqref{eq:direct}. Conditional on the selected predictors, the audit indicators are i.i.d. bounded variables; apply one-sided Hoeffding and a union bound with the Gibbs event.
\end{proof}
The direct disagreement bridge avoids computing an exact infinite-posterior vote and may be tighter than routing through a factor-two vote bound. It is a standard risk comparison, not a new derandomization principle. C-bounds use additional disagreement or second-margin information \citep{germain2015majority}. A posterior mean network is neither an exact posterior vote nor an automatic Gibbs-risk certificate. Appendix~\ref{app:mc} accounts for Monte Carlo evaluation of the source Gibbs risk.

\section{Paired Certification: Proofs and Localized Audits}
\label{app:relative}

\subsection{Full selection and probability argument}
Condition on the pilot that fixes $h_0$, the loss class, $\pi$, and the kernel grid. The normalized signed loss $z_h$ in Theorem~\ref{thm:relative} is then fixed before the certification sample is drawn. PAC-Bayes-kl gives one event of probability at least $1-\delta_\Delta$ on which
\begin{equation}
 \frac{1+D_{P_s}(\rho,h_0)}2\leq p_+(q_\rho^\Delta,C_\rho^\Delta)
\end{equation}
for every posterior. In particular the baseline does not need a second PAC-Bayes complexity penalty: it enters the already-fixed loss function, not an independently selected posterior. If the baseline is chosen on the certification data, this argument no longer applies without additional uniform control.

Write $f_\rho^\Delta=c_\rho^\Delta+g_\rho^\Delta+e_\rho^\Delta$. The exact decomposition is
\begin{align}
 D_{P_t}(\rho,h_0)-D_{P_s}(\rho,h_0)
 ={}&\E_{P_t^X}(f_{t,\rho}^\Delta-f_\rho^\Delta)
 +\langle g_\rho^\Delta,\mu_t-\mu_s\rangle
 +(\E_{P_t^X}-\E_{P_s^X})e_\rho^\Delta.
\end{align}
Its upper bound is $\nu_\rho+B_\rho^\Delta\MMD+2a_\rho^\Delta$. The $J$ embedding events have total failure probability at most $\delta_m$ and do not depend on the posterior. On their intersection with the source event, every eligible pair $(\rho,j)$ is controlled, so selecting a candidate and kernel using these values does not require a further finite-candidate penalty. This does not grant uniformity to separately estimated audit quantities: their witnesses must themselves hold simultaneously over the selection class.

If absolute and paired certificates are both reported as jointly valid, allocate distinct budgets $\delta_s$ and $\delta_\Delta$ to their source events and reuse the common embedding event only once. If only the paired gate is claimed, no absolute-risk source event is needed. Monte Carlo upper estimation of $q_\rho^\Delta$ uses fresh posterior draws as in Appendix~\ref{app:mc}; its correction is inside $p_+$ and is multiplied by the outer factor two when the affine reduction is undone. Repeated rounds require a declared horizon and conditional sampling assumptions with failure budgets summing to the desired overall level; a static theorem is not automatically anytime-valid.

\subsection{Disagreement-localized finite-domain witnesses}
Let $\mathcal X=\{x_1,\ldots,x_d\}$ be the full domain, the labels be binary, and $h_1,\ldots,h_H$ and $h_0$ be pilot-fixed classifiers on this domain. Put $s_{\rho,i}=\Pr_\rho\{h(x_i)=1\}$ and $t_i=h_0(x_i)$. Use the simultaneous cell audit event of Proposition~\ref{prop:audit}. Its estimates $\widehat p_{s,i}$ and radii $e_{s,i}$ give
\begin{align}
 f_\rho^\Delta(x_i)&=(1-2p_{s,i})(s_{\rho,i}-t_i),\\
 v_{\rho,i}^\Delta&=(1-2\widehat p_{s,i})(s_{\rho,i}-t_i),\\
 a_\rho^\Delta&=2\max_i e_{s,i}|s_{\rho,i}-t_i|.
 \label{eq:localized-residual}
\end{align}
Use $v_\rho^\Delta$ in Proposition~\ref{prop:finite-domain} to obtain $c_\rho^\Delta$ and $B_\rho^\Delta$. These witnesses hold simultaneously for all posteriors on the fixed library.

For the conditional term, the exact binary identity is
\begin{equation}
 |f_{t,\rho}^\Delta(x_i)-f_\rho^\Delta(x_i)|
 =2|p_{t,i}-p_{s,i}|\,|s_{\rho,i}-t_i|.
 \label{eq:localized-conditional}
\end{equation}
Let $d_i^+$ be the cellwise conditional-TV upper bound from~\eqref{eq:cell-audit}. On target inputs $U=(X_1,\ldots,X_u)$ independent of the cell audits and pilot, define
\begin{align}
 w_h^+&=\min\left\{1,\frac1u\sum_{b=1}^u d_{i(X_b)}^+
              \ind\{h(X_b)\ne h_0(X_b)\}
              +\sqrt{\frac{\log(H/\delta_w)}{2u}}\right\},\\
 \nu_\rho^{\rm loc}&=2\sum_{h=1}^H\rho(h)w_h^+.
 \label{eq:localized-audit}
\end{align}

\begin{proposition}[Simultaneous localized witnesses]
\label{prop:localized}
With probability at least $1-\delta_p-\delta_w$,~\eqref{eq:localized-residual} and~\eqref{eq:localized-audit} validate~\eqref{eq:contrast-assumption} and~\eqref{eq:relative-conditional} simultaneously for all posteriors on the pilot-fixed binary library. The posterior may therefore be selected using these audit results. If a simultaneously valid global budget $\eta^+$ is also available, use $\nu_\rho=\min\{\nu_\rho^{\rm loc},2\eta^+\}$.
\end{proposition}
\begin{proof}
The difference between the exact and estimated signed source losses is $-2(p_{s,i}-\widehat p_{s,i})(s_{\rho,i}-t_i)$, giving~\eqref{eq:localized-residual} on the cell event. Binary zero--one loss gives~\eqref{eq:localized-conditional} directly. Because $t_i\in\{0,1\}$,
\begin{equation}
 |s_{\rho,i}-t_i|=\sum_{h=1}^H\rho(h)\ind\{h(x_i)\ne h_0(x_i)\}.
\end{equation}
Condition on the pilot and cell audits. For each fixed library member, the $u$ weighted disagreement observations in~\eqref{eq:localized-audit} are i.i.d. in $[0,1]$. One-sided Hoeffding at failure level $\delta_w/H$ and a union bound control every library member; linearity then controls every posterior, even one selected using $U$. Integrating over the conditioning and adding $\delta_p$ proves the result.
\end{proof}
In states where every candidate agrees with $h_0$, the paired residual and true conditional penalty are exactly zero even if the label conditional is unknown there. A function identically equal to the baseline has weighted disagreement zero everywhere and can be assigned the exact bound zero instead of a sampling radius. More generally, the finite-library concentration cost may dominate small differences; localization is not a universal tightness improvement.

\subsection{A strict robust-comparison example}
The distinction in~\eqref{eq:paired-robust} is not merely notational. Consider the full two-state domain $\{0,1\}$, source marginal $(1/2,1/2)$, delta kernel, radius $\varepsilon=1/\sqrt2$, and an unchanged deterministic label channel. The marginal ambiguity ball contains the entire two-state simplex. Specify bounded losses for a baseline and candidate by $\ell_0(x)=0.2+0.6x$ and $\ell_1(x)=0.1+0.6x$. Then
\begin{equation}
 \sup_Q(R_Q(h_1)-R_Q(h_0))=-0.1,
 \qquad \sup_Q R_Q(h_1)-\inf_Q R_Q(h_0)=0.7-0.2=0.5.
\end{equation}
The contrast is the constant $-0.1$, so its centered RKHS norm and residual are zero. These are defined bounded losses in an analytical example, not zero--one classification results. The example illustrates cancellation at the population level; finite-sample confidence can still prevent accepting an update whose improvement is too small to resolve.

\subsection{Acceptance probability in the two-state construction}
\label{app:power}
In Proposition~\ref{prop:example}, the source probability of $Y=1$ is $0.2(0.98)+0.8(0.02)=0.212$, while the candidate risk is $0.02$. Hence $D_{P_s}(h_1,h_0)=-0.192$ and $\E q_\rho^\Delta=0.404$. The normalized paired observations are i.i.d. in $[0,1]$, giving
\begin{equation}
 \Pr(q_\rho^\Delta>0.42)\leq e^{-2n(0.42-0.404)^2}.
\end{equation}
For the delta kernel on two states, $\widehat d_k=\sqrt2|\widehat t-\widehat s|$. Here $\widehat s$ and $\widehat t$ are the means of independent Bernoulli samples of sizes $m$ and $r$. Weighted Hoeffding applied to $(\widehat t-t)-(\widehat s-0.2)$ gives, for any $a>0$,
\begin{equation}
 \Pr\bigl(|(\widehat t-t)-(\widehat s-0.2)|>a\bigr)
 \leq2\exp\left(-\frac{2a^2}{m^{-1}+r^{-1}}\right).
\end{equation}
Since $|t-0.2|\leq0.02$, choosing $a=0.06/\sqrt2-0.02>0$ proves the second failure term in the main text. The two failure events need not be independent; source embedding observations may reuse the certification inputs. Their union bound is less than $0.019062$. This is uniform as a numerical lower bound on power for each fixed $t\in[0.18,0.22]$, not a simultaneous sampling assertion over an uncountable family of experiments.

The exact sufficient-condition calculation is $C_\rho^\Delta=\log(16000)/10000$, $b_k=0.02(1+\sqrt{2\log80})$, and
\begin{equation}
 -0.16+\sqrt{2C_\rho^\Delta}+0.48\sqrt2(0.06+b_k)
 =-0.021501499\ldots.
\end{equation}
The target baseline risk is $0.02+0.96t$, so the true improvement is $0.96t\geq0.1728$. The sufficient-condition calculation deliberately uses the looser Pinsker source endpoint, so the exact inverse-kl endpoint can also pass. No label audit is needed only because the channel is specified analytically. Learning that channel or enlarging the hypothesis library consumes additional information and may remove acceptance. The $0.98$ power bound and $0.95$ confidence level quantify different events; neither is an observed experimental frequency.

\subsection{A source-worsening update certified on the target}
\label{app:directional-example}
Use the full two-state domain, delta kernel, and known channel $P(Y\ne X\mid X)=0.02$, but now set $P_s(X=1)=0.2$ and $P_t(X=1)=0.8$. The baseline is $h_0\equiv0$ and the candidate is $h_1\equiv1$, with uniform prior on these two hypotheses and posterior $\delta_{h_1}$. The baseline is source-optimal in this library, while the candidate is target-optimal. The conditional contrast is $(0.96,-0.96)$, so $c=0$, $B=0.96\sqrt2$, $a=\nu=0$, $D_{P_s}=0.576$, and $D_{P_t}=-0.576$. Consequently the radial certificate cannot strictly accept on its validity event, regardless of sample size.

With the same $n=m=r=10^4$ and $\delta_\Delta=\delta_m=0.025$ as Proposition~\ref{prop:example}, write $\widehat t,\widehat s$ for the input frequencies. Sufficient sample conditions for a directional acceptance are
\begin{equation}
 q_\rho^\Delta\leq0.81,\qquad \widehat t-\widehat s\geq0.56.
\end{equation}
The empirical signed shift is $-1.92(\widehat t-\widehat s)$. Therefore, using the same $C_\rho^\Delta$ and $b_k$ as before,
\begin{equation}
 V_{\rho,1}^{\rm dir}\leq0.62+\sqrt{2C_\rho^\Delta}
 -1.92(0.56)+0.96\sqrt2\,b_k<-0.30.
\end{equation}
The source normalized mean is $0.788$, so the first condition fails with probability at most $e^{-2\cdot10^4(0.022)^2}$. Independent source and target embedding samples give a one-sided weighted Hoeffding bound of $e^{-10^4(0.04)^2}$ for the second condition. A gate evaluating this candidate at margin $0.30$ thus accepts with probability at least
\begin{equation}
 1-e^{-9.68}-e^{-16}>0.9999.
\end{equation}
These analytical statements explain the corresponding repeated-sampling experiment; the observed 1,000/1,000 directional acceptances are reported in Appendix~\ref{app:experiments}. They isolate the value of preserving shift direction under covariate shift with a known channel. They do not establish feasibility with unknown deep-feature conditionals or audited remainders. The radius-only credal set includes the source law, at which the candidate is worse; no valid uniform safe-improvement guarantee over that set can accept it. The directional certificate answers the different, target-specific question.

\section{A Labeled Conditional Audit on a Finite Domain}
\label{app:audit}
This construction addresses a restricted but identifiable setting. Suppose $\mathcal X=\{x_1,\ldots,x_d\}$ is the full domain and $\mathcal Y=\{0,1\}$. Let $p_{D,i}=P_D(Y=1\mid X=x_i)$ for $D\in\{s,t\}$. Obtain independent i.i.d. labeled source and target audit samples, separate from the pilot, certification, and adaptation samples. Conditional on the observed cell counts $N_{D,i}$, let $\widehat p_{D,i}$ be the within-cell label mean when the count is positive. For zero count set $\widehat p_{D,i}=1/2$. Define
\begin{equation}
 e_{D,i}=\begin{cases}
 \min\{1,\sqrt{\log(4d/\delta_p)/(2N_{D,i})}\},&N_{D,i}>0,\\
 1,&N_{D,i}=0,
 \end{cases}\qquad
 d_i^+=\min\{1,|\widehat p_{s,i}-\widehat p_{t,i}|+e_{s,i}+e_{t,i}\}.
 \label{eq:cell-audit}
\end{equation}

\begin{proposition}[Binary conditional audit]
\label{prop:audit}
With probability at least $1-\delta_p$, $|p_{s,i}-p_{t,i}|\leq d_i^+$ simultaneously for all cells. On a further independent target-input sample $U=(X_1,\ldots,X_u)$,
\begin{equation}
 \eta^+=\min\left\{1,\frac1u\sum_{b=1}^u d_{i(X_b)}^+
                       +\sqrt{\frac{\log(1/\delta_u)}{2u}}\right\}
 \label{eq:eta-audit}
\end{equation}
upper-bounds $\E_{P_t^X}\TV(P_t(\cdot\mid X),p_s(\cdot\mid X))$ with probability at least $1-\delta_p-\delta_u$.
\end{proposition}
\begin{proof}
Conditional on cell memberships, each positive-count cell contains i.i.d. Bernoulli labels. The two-sided Hoeffding failure probability for each of the $2d$ cells is at most $\delta_p/(2d)$; zero-count cells are covered deterministically. Integrating out the counts preserves the union bound. Binary total variation equals $|p_{s,i}-p_{t,i}|$, and the triangle inequality bounds this by $d_i^+$. Conditional on the audits, $d_{i(X_b)}^+\in[0,1]$ are i.i.d. on the independent input sample, so one-sided Hoeffding yields~\eqref{eq:eta-audit}.
\end{proof}
Unobserved source states retain the trivial bound, consistent with the support issue in Section~\ref{sec:setting}. The guarantee applies to the audited target law, not arbitrary future conditional shifts. For multiple target laws or selected partitions, use a further split or an explicit simultaneous allocation.

\paragraph{A simultaneous source-loss residual.}
For binary classifiers on these same states, let $s_{\rho,i}=\Pr_{h\sim\rho}\{h(x_i)=1\}$, evaluated exactly for a finite ensemble. Then
\begin{equation}
 f_\rho(x_i)=p_{s,i}+(1-2p_{s,i})s_{\rho,i},\qquad
 v_{\rho,i}=\widehat p_{s,i}+(1-2\widehat p_{s,i})s_{\rho,i}.
\end{equation}
On the cell audit event,
\begin{equation}
 |f_\rho(x_i)-v_{\rho,i}|=|p_{s,i}-\widehat p_{s,i}|\,|1-2s_{\rho,i}|
 \leq e_{s,i}.
\end{equation}
Thus Proposition~\ref{prop:finite-domain} gives computable $B_\rho$ and $a_\rho=\max_i e_{s,i}$ simultaneously for all posteriors on the fixed class. This same event can support both residual and conditional bounds and should be charged only once in the failure ledger. Sparse cells can make the certificate vacuous; deleting them after observing the audits is not justified.

\paragraph{Quantized representations.}
A pilot-fixed map $Z=q(f_0(X))$ into finitely many cells permits application of this construction to the distribution of $(Z,Y)$, provided every certified classifier is a function of $Z$. The resulting bound certifies that quantized predictor and the induced conditional $P(Y\mid Z)$, not an unrestricted network on $X$. Even if $P_s(Y\mid X)=P_t(Y\mid X)$, it need not follow that $P_s(Y\mid Z)=P_t(Y\mid Z)$; the finite-cell conditional audit must still be performed or its budget assumed explicitly.

\section{Residuals, Kernel Scaling, and Bounded Surrogates}
\label{app:loss}
For a KRR fit with $\lambda>0$, $\alpha=(K+\lambda I)^{-1}v$ and $g(\cdot)=\sum_i\alpha_i k(x_i,\cdot)$,
\begin{equation}
 \|g\|_\cH^2=v^\top(K+\lambda I)^{-1}K(K+\lambda I)^{-1}v
 \leq\frac{\|v\|_2^2}{4\lambda},\qquad
 K\alpha-v=-\lambda(K+\lambda I)^{-1}v.
\end{equation}
Diagonalizing $K$ proves the norm inequality because $\mu/(\mu+\lambda)^2\leq1/(4\lambda)$ for every $\mu\geq0$. On a genuinely finite full domain, a known loss-vector error $e_i$ gives a valid residual $a\leq\max_i(|(K\alpha)_i-v_i|+e_i)$. On a continuous domain this is only a sample residual, not a bound between sampled inputs.

Kernel amplitude cannot improve a certificate by itself: replacing $k$ by $\widetilde k=t k$ with $t>0$ multiplies MMD, $\kappa$, and the embedding radius by $\sqrt t$, while dividing the norm of the same function by $\sqrt t$. Their products are invariant and the residual is unchanged. Bandwidth changes are not mere rescalings and require recomputing all terms.

For a multiclass probabilistic predictor with true-class probability $p_y$, define normalized clipped cross-entropy $\ell_M=\min\{-\log p_y,M\}/M$, $M>0$. It is in $[0,1]$. With a fixed argmax tie rule, an incorrect prediction implies $p_y\leq1/2$, so
\begin{equation}
 \ind\{\arg\max_c p_c\ne y\}\leq
 \frac{M}{\min\{M,\log2\}}\ell_M.
\end{equation}
This supplies an explicit, possibly loose, Gibbs classification-risk conversion. Boundedness alone does not establish Assumption~\ref{ass:rkhs}. In particular, a Lipschitz predictor composed with a feature map does not automatically have finite Gaussian-RKHS loss norm. A uniform approximation certificate must control the conditional loss, including uncertainty in the source label conditional, over the full admissible domain.

\section{A Restricted Continuous Conditional-Loss Certificate}
\label{app:continuous}
The following one-dimensional RKHS calculation supplies an alternative to a finite domain. It is an explicit kernel-space construction, not a claim of a new RKHS characterization or automatic applicability to neural logits.

\begin{proposition}[Exponential kernel on an interval]
\label{prop:exponential}
For $\alpha>0$ and $k_\alpha(z,z')=\exp(-\alpha|z-z'|)$ on $[0,1]$, the RKHS consists of the absolutely continuous representatives of $H^1([0,1])$, with squared norm
\begin{equation}
 \|g\|_{\cH_\alpha}^2=
 \frac{g(0)^2+g(1)^2}{2}
 +\frac1{2\alpha}\int_0^1 g'(z)^2\,dz
 +\frac\alpha2\int_0^1 g(z)^2\,dz.
 \label{eq:exponential-norm}
\end{equation}
Consequently, $\|g\|_\infty\leq M$ and $L$-Lipschitz continuity imply
\begin{equation}
 \|g\|_{\cH_\alpha}^2\leq M^2+\frac{L^2+\alpha^2M^2}{2\alpha}.
 \label{eq:exponential-upper}
\end{equation}
\end{proposition}
\begin{proof}
Polarize the right side of~\eqref{eq:exponential-norm} to define an inner product $A$. Its norm is equivalent to the usual $H^1$ norm, so this space is complete and point evaluation is continuous. For $z\in(0,1)$ put $k_z(u)=e^{-\alpha|u-z|}$. Away from $z$, $k_z''=\alpha^2k_z$; its first derivative has jump $-2\alpha$ at $z$. Integrating $g'k_z'$ separately over $[0,z]$ and $[z,1]$ yields
\begin{equation}
 \int_0^1g'k_z'
 =-\alpha g(1)k_z(1)-\alpha g(0)k_z(0)
   +2\alpha g(z)-\alpha^2\int_0^1gk_z.
\end{equation}
Substitution gives $A(g,k_z)=g(z)$. The same integration at an endpoint gives reproduction there, or follows by continuity. Uniqueness of the RKHS proves~\eqref{eq:exponential-norm}. A Lipschitz function is absolutely continuous with $|g'|\leq L$ almost everywhere, giving~\eqref{eq:exponential-upper}. No Robin boundary conditions are imposed on arbitrary members $g$: boundary terms belong to the inner product. Applying a differential operator as if every $H^1$ function were in its strong operator domain would be invalid.
\end{proof}

\paragraph{The condition that must actually be certified.}
Analyze the joint law of $(Z,Y)$, where $Z\in[0,1]$ is pilot-fixed and every predictor and its loss depend on the input only through $Z$. For an eligible posterior define
\begin{equation}
 u_\rho(z,y)=\E_\rho[\ell(h,z,y)-\ell(h_0,z,y)].
\end{equation}
Suppose $u_\rho(\cdot,y)$ is $L_u$-Lipschitz uniformly in $y$ and the \emph{source conditional on this representation} satisfies
\begin{equation}
 \TV(p_s(\cdot\mid z),p_s(\cdot\mid z'))\leq L_p|z-z'|.
\end{equation}
These are global assumptions or validated upper bounds, not estimates from a plot. Adding and subtracting the expectation of $u_\rho(z',Y)$ under $p_s(\cdot\mid z)$ gives
\begin{equation}
 |f_\rho^\Delta(z)-f_\rho^\Delta(z')|
 \leq(L_u+2L_p)|z-z'|.
\end{equation}
The factor two is the labelwise oscillation bound for a signed loss in $[-1,1]$. Thus $M=1$, $L=L_u+2L_p$, $c=0$, and $a=0$ give a sufficient paired witness via~\eqref{eq:exponential-upper}. A validated centering or smaller oscillation can improve this bound. For an ordinary loss in $[0,1]$, the analogous constant is $L_{\ell}+L_p$. The target conditional remainder still requires separate control; source regularity does not bound target conditional shift.

\paragraph{Counterexample to smooth-logit certification.}
Let $p_s(Y=1\mid z)=\ind\{z\geq1/2\}$ and let
\begin{equation}
 \ell(z,y)=\frac{\log(1+e^z)-yz}{2},\qquad 0\leq z\leq1.
\end{equation}
This is bounded in $[0,1]$ and smooth for each label (equivalently, cross-entropy clipped at $2$ and divided by $2$, with clipping inactive here). Nevertheless, its conditional expectation jumps downward by $1/4$ at $z=1/2$ and is not in this continuous-function RKHS. A Lipschitz logit or bounded surrogate alone therefore does not supply the required witness. This obstruction also applies when a frozen representation makes its label conditional irregular.

\section{Direct Labeled-Target Comparison}
\label{app:direct-audit}
When target labels are used to estimate conditional shift, a directly audited excess risk is an essential comparator.
\begin{proposition}[Finite-library direct audit]
\label{prop:direct-audit}
Let a baseline and a library of $H$ classifiers be fixed by the pilot. On $u$ i.i.d. labeled target observations independent of that pilot define $\overline d_h=u^{-1}\sum_b[\ell(h,X_b,Y_b)-\ell(h_0,X_b,Y_b)]$. With probability at least $1-\delta_T$, simultaneously for every distribution $\rho$ on the library,
\begin{equation}
 D_{P_t}(\rho,h_0)\leq
 A_\rho:=\sum_h\rho(h)\overline d_h+
 \sqrt{\frac{2\log(H/\delta_T)}u}.
 \label{eq:direct-audit}
\end{equation}
The posterior may be selected using this audit.
\end{proposition}
\begin{proof}
Each signed observation has range length two. One-sided Hoeffding at level $\delta_T/H$ bounds each fixed classifier's target contrast by its empirical contrast plus the displayed radius. A union bound and averaging over any $\rho$ prove the assertion.
\end{proof}
This elementary comparator needs no RKHS norm, approximation residual, or conditional-shift budget. It does need labeled target observations and a pilot-fixed library. Its labels can be shared with a conditional audit if the requirements of each event hold; a union bound does not require the events to be independent. Using $\min\{A_\rho,V_{\rho,j}\}$ for deployment is valid only after allocating failure probability to both certificates. Direct audits may dominate transfer in some regimes. Demonstrating an advantage requires matching target-label and source-information budgets, not giving our method privileged labels. An unlabeled-only setting must report this comparator as unavailable, not as an unfairly weakened baseline.

\section{Monte Carlo, Priors, and Probability Accounting}
\label{app:mc}
Suppose a posterior has been selected before drawing fresh independent $H_1,\ldots,H_K\sim\rho$. On a fixed source certification sample define
\begin{equation}
 Z_b=\frac1n\sum_{i=1}^n\ell(H_b,X_i,Y_i),\qquad
 \widetilde q=\frac1K\sum_{b=1}^K Z_b,\qquad
 q^+=\min\{1,\widetilde q+\sqrt{\log(1/\delta_{\rm MC})/(2K)}\}.
\end{equation}
Conditional on the selected posterior and source data, the $Z_b$ are i.i.d. in $[0,1]$ with mean $q_\rho$. Hence $q_\rho\leq q^+$ with probability at least $1-\delta_{\rm MC}$. Since $p_+(q,C)$ is nondecreasing in $q$, replacing $q_\rho$ by $q^+$ in an upper certificate is valid. For a lower endpoint use the analogous lower confidence bound. A two-sided Monte Carlo interval replaces $\log(1/\delta_{\rm MC})$ by $\log(2/\delta_{\rm MC})$.

Reusing each sampled network across all $n$ examples leaves only $K$ independent Monte Carlo replicates; treating the $nK$ losses as independent is incorrect. Optimizing the posterior on these same draws requires a different uniform argument or fresh evaluation draws. Exact summation for a finite hypothesis class eliminates this Monte Carlo term entirely. Approximating a kernel with random features likewise changes the certificate unless a fixed random-feature kernel is the declared kernel or an approximation-error bound is added.

For $\pi=\mathcal N(\theta_0,\sigma^2I)$ and $\rho=\mathcal N(\theta_1,\sigma^2I)$,
\begin{equation}
 \KL(\rho\|\pi)=\frac{\|\theta_1-\theta_0\|_2^2}{2\sigma^2}.
\end{equation}
The mean $\theta_0$ and prior variance must be fixed independently of $S$ under the stated theorem. Selecting a posterior is permitted by uniformity; selecting a prior variance on $S$ is not the same operation. A finite grid of priors can be handled with an additional union bound. A Dirac posterior has infinite KL relative to a nondegenerate Gaussian prior. By contrast, a Dirac mass on a finite-class hypothesis with positive prior mass has finite KL.

When jointly reporting absolute and paired finite-kernel certificates, localized finite-cell audits, and representative disagreement audits, a sufficient overall allocation is
\begin{equation}
 \delta_s+\delta_\Delta+\delta_++\delta_-+\delta_m+\delta_p+\delta_u+\delta_w
 +\delta_{\rm MC}+\delta_D+\delta_T\leq\delta_{\rm total}.
\end{equation}
Terms corresponding to unused procedures are omitted; their failure allocation is zero and their formulas containing logarithms are not evaluated. A gate using only the affine paired source event does not spend $\delta_s,\delta_+,\delta_-$. If it selects between affine and split-loss bounds, all three source allocations are positive. Allocate $\delta_{\rm MC}$ across all Monte Carlo quantities if more than one is estimated. For~\eqref{eq:parts}, use an upper Monte Carlo bound for $q_\rho^+$ and a lower bound for $q_\rho^-$; the signs cannot be interchanged. The common embedding and cell events are charged only once. All candidates selected using an audit must be covered simultaneously, or evaluated on a fresh audit. Uniformity over PAC-Bayes posteriors does not turn an arbitrary post-selection Hoeffding calculation into a valid audit.

\section{Auxiliary Anchor-Distance Geometry}
\label{app:geometry}
\begin{proposition}[Anchor-distance profile]
\label{prop:anchor}
For a fixed kernel with $k(x,x)=c\leq\kappa^2$, let $d_k(x,z)=\|\phi(x)-\phi(z)\|_\cH$. Simultaneously for all anchors $z$,
\begin{equation}
 |\E_{P_s^X}d_k^2(z,X)-\E_{P_t^X}d_k^2(z,X)|
 \leq2\kappa\MMD_k(P_s^X,P_t^X).
 \label{eq:anchor}
\end{equation}
\end{proposition}
\begin{proof}
Since $d_k^2(z,x)=2c-2k(z,x)$, the left side equals $2|\langle\phi(z),\mu_s-\mu_t\rangle|$. Apply Cauchy--Schwarz.
\end{proof}
This controls an expected squared RKHS distance, not individual distances, geodesics, nearest neighbors, margins, or class accuracy. The same argument applies within a class if its \emph{true class-conditional} MMD is controlled. A marginal MMD bound alone supplies no rare-class guarantee. Pseudo-class conditioning needs a separate pseudo-label error argument.

%% file: experiment_appendix.tex
\section{Implementation}\label{sec:impl}
For a posterior $Q$ over a finite library of state-constant classifiers and a fixed baseline $h_0$, let $g=\mathbb{E}_Q[h]-h_0$ (zero outside the disagreement region $G$). Under the delta kernel on $d$ states,
\[
\mathcal{E}_t(Q)=\underbrace{\mathcal{E}_s(Q)}_{S}+\underbrace{\langle p_t-p_s,\,w\rangle}_{T}+\underbrace{\textstyle\sum_i p_t(i)\,2g_i\,(\eta_s(i)-\eta_t(i))}_{C},\qquad w=g\,(1-2\eta_s).
\]
\begin{itemize}
\item $S$: PAC-Bayes-kl (Maurer) bound on the paired contrast, uniform prior on the deduplicated library; affine, positive/negative split, or confidence-corrected minimum (default: minimum).
\item $T$ radial: $\overline{\MMD}\,(\|\hat w\|+\|\rho\|)$ with $\overline{\MMD}=\|\hat p_r-\hat p_m\|+\varepsilon_r+\varepsilon_m$.
\item $T$ directional: $\langle\hat p_u-\hat p_m,\hat w\rangle+\sum_i(h_u+h_m)_i|\hat w_i|+\sum_i(|\hat p_u-\hat p_m|+h_u+h_m)_i\rho_i$, where $h$ are simultaneous per-state kl half-widths of each marginal. Joint: minimum of the two with all events charged.
\item $\rho$: witness residual from per-state kl audit intervals of $\eta_s$, multiplied by $2|g_i|$ (localized) or $2$ (global).
\item $C$: zero under a known or declared-unchanged channel; exact plug-in under a known shifted channel; otherwise a sign-aware interval bound from source and target audits, localized to $G$ or global.
\item Total failure budget $\delta=0.05$: half to $S$, the rest split equally over the events actually used. Accept iff $S+T+C\le-\text{margin}$.
\end{itemize}
Comparators: never/always adapt; direct labeled-target audit (PAC-Bayes-kl over the library using only target labels); direct conditional-model bound (plug-in channel with Hoeffding deviations on the target integration sample); separate absolute-risk comparison (Exp.~1).
Every row of every experiment stores the full ledger and the exact target excess risk, so every acceptance is scored against ground truth.

\paragraph{Reproducibility and compute.}
Experiments 1--4 and 6 ran on CPUs; the vision diagnostic ran on institutional A100 GPUs. This arXiv source bundle contains the TeX manuscript, bibliography, and appendices, but not the experimental code, environment specification, Slurm scripts, raw per-repetition rows, or analysis scripts. Those materials and asset-license details must be supplied separately for independent reproduction.

\begin{table}[H]\centering\small
\caption{Numerical consistency audit: valid candidate--certificate evaluations and the number whose computed bound fell below the true target excess risk. Evaluations are correlated, so this is neither a proof nor a coverage estimate. Arms that assume an unchanged channel where it shifted are excluded as invalid controls.}
\begin{adjustbox}{max width=\linewidth}\begin{tabular}{lrr}\toprule
Experiment & evaluations & violations\\\midrule
Exp 1 & 192,000 & 0 \\
Exp 2 & 29,000 & 0 \\
Exp 3 & 35,100 & 0 \\
Exp 6 & 132,800 & 0 \\
Exp 4 (k8) & 10,295 & 0 \\
Exp 4 (k32) & 8,900 & 0 \\
Exp 4 (k32 rich) & 8,630 & 0 \\
\bottomrule\end{tabular}\end{adjustbox}\end{table}

\section{Experiment 1: two-state negative certificate}
Binary inputs, known channel $P(Y\neq X\mid X)=0.02$, Dirac posterior on the candidate, $n=m=r=10{,}000$, 1{,}000 repetitions.
\begin{table}[H]\centering\small
\caption{Acceptance (\%) at $n=m=r=10{,}000$. Margins: 0.02 (Case A), 0.30 (Case B). $\mathcal{E}_s,\mathcal{E}_t$: true source and target excess risk of the candidate.}
\begin{adjustbox}{max width=\linewidth}\begin{tabular}{lrr rrrrrrr}\toprule
Setting & $\mathcal{E}_s$ & $\mathcal{E}_t$ & Radial & Directional & Joint & Abs.\ radial & Abs.\ dir. & Direct cond. & Always\\\midrule
A, $p_t=0.18$ & -0.19 & -0.17 & 100 & 100 & 100 & 100 & 100 & 100 & 100 \\
A, $p_t=0.20$ & -0.19 & -0.19 & 100 & 100 & 100 & 100 & 100 & 100 & 100 \\
A, $p_t=0.22$ & -0.19 & -0.21 & 100 & 100 & 100 & 100 & 100 & 100 & 100 \\
B beneficial, $p_t=0.8$ & +0.58 & -0.58 & 0 & 100 & 100 & 0 & 100 & 100 & 100 \\
B neutral, $p_t=0.5$ & +0.58 & +0.00 & 0 & 0 & 0 & 0 & 0 & 0 & 100 \\
B harmful, $p_t=0.2$ & +0.58 & +0.58 & 0 & 0 & 0 & 0 & 0 & 0 & 100 \\
\bottomrule\end{tabular}\end{adjustbox}\end{table}

\begin{table}[H]\centering\small
\caption{Complete ledger for the Case B beneficial update (source-worsening, target-improving), means over 1{,}000 repetitions. Acceptance with 95\% Clopper--Pearson interval.}
\begin{adjustbox}{max width=\linewidth}\begin{tabular}{lrr}\toprule
Component & Radial & Directional\\\midrule
empirical source contrast $\hat c_s$ & +0.576 & +0.576 \\
source term $S$ (PAC-Bayes-kl, min) & +0.612 & +0.612 \\
$\widehat{\mathrm{MMD}}$ & +0.848 & +0.848 \\
$\|\hat w\|$ & +1.358 & +1.358 \\
radial transfer $\overline{\mathrm{MMD}}\,(\|\hat w\|+\|\rho\|)$ & +1.205 & -- \\
directional inner product $\langle \hat p_t-\hat p_s,\hat w\rangle$ & -- & -1.152 \\
directional concentration slack & -- & +0.053 \\
directional transfer $T$ & -- & -1.099 \\
conditional penalty $C$ (known channel) & +0.000 & +0.000 \\
certificate $S+T+C$ & +1.817 & -0.486 \\
true target excess risk & -0.576 & -0.576 \\
true transfer term & -1.152 & -1.152 \\
acceptance at margin 0.30 (\%) & 0 [0,0.37] & 100 [99.63,100] \\
\bottomrule\end{tabular}\end{adjustbox}\end{table}

\begin{table}[H]\centering\small
\caption{Acceptance (\%) against sample size.}
\begin{adjustbox}{max width=\linewidth}\begin{tabular}{r rrrr rrrr}\toprule
& \multicolumn{4}{c}{Case A, $p_t=0.20$} & \multicolumn{4}{c}{Case B beneficial}\\\cmidrule(lr){2-5}\cmidrule(lr){6-9}
$n=m=r$ & Rad. & Dir. & Abs.\ dir. & Direct & Rad. & Dir. & Abs.\ dir. & Direct\\\midrule
500 & 0 & 4 & 0 & 0 & 0 & 3 & 0 & 84 \\
2,000 & 96 & 100 & 6 & 100 & 0 & 99 & 1 & 100 \\
10,000 & 100 & 100 & 100 & 100 & 0 & 100 & 100 & 100 \\
50,000 & 100 & 100 & 100 & 100 & 0 & 100 & 100 & 100 \\
\bottomrule\end{tabular}\end{adjustbox}\end{table}
\begin{figure}[H]\centering
\begin{adjustbox}{max width=\linewidth}
\begin{tikzpicture}\begin{groupplot}[group style={group size=3 by 1, horizontal sep=0.7cm, yticklabels at=edge left}, width=0.305\linewidth, height=6.1cm, xmode=log, ymin=-0.05, ymax=1.05, xlabel={$n=m=r$}]
\nextgroupplot[title={Case A, $p_t=0.20$}, ylabel={acceptance},
 legend style={figure legend,font=\scriptsize,at={(0,1.38)},anchor=south west,legend columns=3,
 /tikz/every even column/.append style={column sep=.7em}}]
\addplot[m_cert_radial, eb] coordinates {(449.6,0) += (0,0.00368208) -= (0,0) (1798.4,0.955) += (0,0.0119904) -= (0,0.0147528) (8992,1) += (0,0) -= (0,0.00368208) (44960,1) += (0,0) -= (0,0.00368208)};\addlegendentry{radial}

\addplot[m_cert_directional, eb] coordinates {(435.2,0.042) += (0,0.0143503) -= (0,0.0115653) (1740.8,1) += (0,0) -= (0,0.00368208) (8704,1) += (0,0) -= (0,0.00368208) (43520,1) += (0,0) -= (0,0.00368208)};\addlegendentry{directional}

\addplot[m_cert_joint, eb] coordinates {(464,0.033) += (0,0.013034) -= (0,0.0101776) (1856,1) += (0,0) -= (0,0.00368208) (9280,1) += (0,0) -= (0,0.00368208) (46400,1) += (0,0) -= (0,0.00368208)};\addlegendentry{joint}

\addplot[m_absolute_radial, eb] coordinates {(550.4,0) += (0,0.00368208) -= (0,0) (2201.6,0.001) += (0,0.00455892) -= (0,0.000974683) (11008,1) += (0,0) -= (0,0.00368208) (55040,1) += (0,0) -= (0,0.00368208)};\addlegendentry{absolute-risk, radial}

\addplot[m_absolute_directional, eb] coordinates {(536,0) += (0,0.00368208) -= (0,0) (2144,0.065) += (0,0.0171011) -= (0,0.0144797) (10720,1) += (0,0) -= (0,0.00368208) (53600,1) += (0,0) -= (0,0.00368208)};\addlegendentry{absolute-risk, directional}

\addplot[m_direct_conditional_model, eb] coordinates {(492.8,0.001) += (0,0.00455892) -= (0,0.000974683) (1971.2,1) += (0,0) -= (0,0.00368208) (9856,1) += (0,0) -= (0,0.00368208) (49280,1) += (0,0) -= (0,0.00368208)};\addlegendentry{conditional model}

\nextgroupplot[title={Case B beneficial, $p_t=0.8$}]
\addplot[m_cert_radial, eb] coordinates {(449.6,0) += (0,0.00368208) -= (0,0) (1798.4,0) += (0,0.00368208) -= (0,0) (8992,0) += (0,0.00368208) -= (0,0) (44960,0) += (0,0.00368208) -= (0,0)};

\addplot[m_cert_directional, eb] coordinates {(435.2,0.03) += (0,0.0125514) -= (0,0.00966951) (1740.8,0.991) += (0,0.0048766) -= (0,0.00801578) (8704,1) += (0,0) -= (0,0.00368208) (43520,1) += (0,0) -= (0,0.00368208)};

\addplot[m_cert_joint, eb] coordinates {(464,0.02) += (0,0.01072) -= (0,0.00774173) (1856,0.987) += (0,0.00606038) -= (0,0.0091278) (9280,1) += (0,0) -= (0,0.00368208) (46400,1) += (0,0) -= (0,0.00368208)};

\addplot[m_absolute_radial, eb] coordinates {(550.4,0) += (0,0.00368208) -= (0,0) (2201.6,0) += (0,0.00368208) -= (0,0) (11008,0) += (0,0.00368208) -= (0,0) (55040,0) += (0,0.00368208) -= (0,0)};

\addplot[m_absolute_directional, eb] coordinates {(536,0) += (0,0.00368208) -= (0,0) (2144,0.009) += (0,0.00801578) -= (0,0.0048766) (10720,1) += (0,0) -= (0,0.00368208) (53600,1) += (0,0) -= (0,0.00368208)};

\addplot[m_direct_conditional_model, eb] coordinates {(492.8,0.845) += (0,0.0218891) -= (0,0.0239415) (1971.2,1) += (0,0) -= (0,0.00368208) (9856,1) += (0,0) -= (0,0.00368208) (49280,1) += (0,0) -= (0,0.00368208)};

\nextgroupplot[title={Case B harmful, $p_t=0.2$}]
\addplot[m_cert_radial, eb] coordinates {(449.6,0) += (0,0.00368208) -= (0,0) (1798.4,0) += (0,0.00368208) -= (0,0) (8992,0) += (0,0.00368208) -= (0,0) (44960,0) += (0,0.00368208) -= (0,0)};

\addplot[m_cert_directional, eb] coordinates {(435.2,0) += (0,0.00368208) -= (0,0) (1740.8,0) += (0,0.00368208) -= (0,0) (8704,0) += (0,0.00368208) -= (0,0) (43520,0) += (0,0.00368208) -= (0,0)};

\addplot[m_cert_joint, eb] coordinates {(464,0) += (0,0.00368208) -= (0,0) (1856,0) += (0,0.00368208) -= (0,0) (9280,0) += (0,0.00368208) -= (0,0) (46400,0) += (0,0.00368208) -= (0,0)};

\addplot[m_absolute_radial, eb] coordinates {(550.4,0) += (0,0.00368208) -= (0,0) (2201.6,0) += (0,0.00368208) -= (0,0) (11008,0) += (0,0.00368208) -= (0,0) (55040,0) += (0,0.00368208) -= (0,0)};

\addplot[m_absolute_directional, eb] coordinates {(536,0) += (0,0.00368208) -= (0,0) (2144,0) += (0,0.00368208) -= (0,0) (10720,0) += (0,0.00368208) -= (0,0) (53600,0) += (0,0.00368208) -= (0,0)};

\addplot[m_direct_conditional_model, eb] coordinates {(492.8,0) += (0,0.00368208) -= (0,0) (1971.2,0) += (0,0.00368208) -= (0,0) (9856,0) += (0,0.00368208) -= (0,0) (49280,0) += (0,0.00368208) -= (0,0)};

\end{groupplot}\end{tikzpicture}
\end{adjustbox}
\caption{Experiment 1: acceptance frequency against sample size, with 95\% exact binomial intervals.}\end{figure}

\section{Experiment 2: finite-state candidate selection}
$d=8$, threshold library and complements (deduplicated), $h_0$ chosen on a 1{,}024-draw pilot, $n=m=r=10{,}000$, $u=20{,}000$, source audit 10{,}000, 100 repetitions. With zero target labels the unchanged-channel assumption is declared (it holds here).
\begin{table}[H]\centering\small
\caption{Decisions with no target labels (candidate family: best of Dirac and Gibbs mixtures). ``Best'' is the best achievable target excess in the library. Direct audit columns use 2{,}000 and 10{,}000 target labels.}
\begin{adjustbox}{max width=\linewidth}\begin{tabular}{l r rrr rr r rr}\toprule
Target & Best & Rad. & Dir. & Joint & Dir.\ deployed & Dir.\ bound & Direct cond. & Audit 2k & Audit 10k\\
& & \multicolumn{3}{c}{acceptance (\%)} & excess & & acc.\ (\%) & \multicolumn{2}{c}{acc.\ (\%)}\\\midrule
concentrated $\lambda=0.9$ & -0.800 & 0 & 100 & 100 & -0.800 & -0.611 & 100 & 100 & 100 \\
concentrated $\lambda=0.5$ & -0.401 & 0 & 100 & 100 & -0.401 & -0.237 & 100 & 100 & 100 \\
moderate $\beta=2$ & -0.098 & 0 & 0 & 0 & +0.000 & +0.038 & 100 & 68 & 100 \\
moderate $\beta=0$ & +0.000 & 0 & 0 & 0 & +0.000 & +0.044 & 0 & 0 & 0 \\
moderate $\beta=-2$ & +0.000 & 0 & 0 & 0 & +0.000 & +0.001 & 0 & 0 & 0 \\
\bottomrule\end{tabular}\end{adjustbox}\end{table}
No harmful acceptance occurred for any certificate or comparator in Experiment 2. With target labels and the fully audited channel, the certificate behaves the same on the concentrated shifts and still does not accept under $\beta=2$ (Figure~\ref{fig:e2}).

\begin{table}[H]\centering\small
\caption{Dirac versus Gibbs-mixture candidates (directional certificate, no target labels).}
\begin{adjustbox}{max width=\linewidth}\begin{tabular}{llrrr}\toprule
Target & Family & Acceptance (\%) & Mean bound & Deployed excess\\\midrule
concentrated $\lambda=0.9$ & Dirac & 100 & -0.611 & -0.800 \\
concentrated $\lambda=0.9$ & Gibbs mixture & 100 & -0.592 & -0.795 \\
concentrated $\lambda=0.5$ & Dirac & 100 & -0.237 & -0.401 \\
concentrated $\lambda=0.5$ & Gibbs mixture & 100 & -0.234 & -0.399 \\
moderate $\beta=2$ & Dirac & 0 & +0.038 & +0.000 \\
moderate $\beta=2$ & Gibbs mixture & 0 & +0.038 & +0.000 \\
\bottomrule\end{tabular}\end{adjustbox}\end{table}
\begin{figure}[H]\centering
\begin{adjustbox}{max width=\linewidth}
\begin{tikzpicture}\begin{groupplot}[group style={group size=3 by 1, horizontal sep=0.7cm, yticklabels at=edge left}, width=0.305\linewidth, height=6.2cm, ybar=0pt, /pgf/bar width=6pt, ymin=0, ymax=1.08, xtick={0,1,2,3,4}, xticklabels={$\lambda{=}0.5$,$\lambda{=}0.9$,$\beta{=}{-}2$,$\beta{=}0$,$\beta{=}2$}, x tick label style={rotate=25, anchor=north east}, enlarge x limits=0.15, grid=none, ymajorgrids]
\nextgroupplot[title={acceptance}, legend style={figure legend,font=\scriptsize,
 at={(0,1.30)},anchor=south west,legend columns=3,
 /tikz/every even column/.append style={column sep=.8em}}]
\addplot[fill=corange, draw=none, bar shift=-5pt, forget plot, eb, error bar style={color=black!70}] coordinates {(0,0) += (0,0.0362167) -= (0,0) (1,0) += (0,0.0362167) -= (0,0) (2,0) += (0,0.0362167) -= (0,0) (3,0) += (0,0.0362167) -= (0,0) (4,0) += (0,0.0362167) -= (0,0)};
\addlegendimage{area legend,fill=corange,draw=none}\addlegendentry{radial}

\addplot[fill=cblue, draw=none, bar shift=0pt, forget plot, eb, error bar style={color=black!70}] coordinates {(0,1) += (0,0) -= (0,0.0362167) (1,1) += (0,0) -= (0,0.0362167) (2,0) += (0,0.0362167) -= (0,0) (3,0) += (0,0.0362167) -= (0,0) (4,0) += (0,0.0362167) -= (0,0)};
\addlegendimage{area legend,fill=cblue,draw=none}\addlegendentry{directional}

\addplot[fill=caqua, draw=none, bar shift=5pt, forget plot, eb, error bar style={color=black!70}] coordinates {(0,1) += (0,0) -= (0,0.0362167) (1,1) += (0,0) -= (0,0.0362167) (2,0) += (0,0.0362167) -= (0,0) (3,0) += (0,0.0362167) -= (0,0) (4,0) += (0,0.0362167) -= (0,0)};
\addlegendimage{area legend,fill=caqua,draw=none}\addlegendentry{joint}

\nextgroupplot[title={useful-update capture}]
\addplot[fill=corange, draw=none, bar shift=-5pt, eb, error bar style={color=black!70}] coordinates {(0,0) += (0,0.0362167) -= (0,0) (1,0) += (0,0.0362167) -= (0,0) (2,0) += (0,0.0362167) -= (0,0) (3,0) += (0,0.0362167) -= (0,0) (4,0) += (0,0.0362167) -= (0,0)};

\addplot[fill=cblue, draw=none, bar shift=0pt, eb, error bar style={color=black!70}] coordinates {(0,1) += (0,0) -= (0,0.0362167) (1,1) += (0,0) -= (0,0.0362167) (2,0) += (0,0.0362167) -= (0,0) (3,0) += (0,0.0362167) -= (0,0) (4,0) += (0,0.0362167) -= (0,0)};

\addplot[fill=caqua, draw=none, bar shift=5pt, eb, error bar style={color=black!70}] coordinates {(0,1) += (0,0) -= (0,0.0362167) (1,1) += (0,0) -= (0,0.0362167) (2,0) += (0,0.0362167) -= (0,0) (3,0) += (0,0.0362167) -= (0,0) (4,0) += (0,0.0362167) -= (0,0)};

\nextgroupplot[title={harmful acceptance}]
\addplot[fill=corange, draw=none, bar shift=-5pt, eb, error bar style={color=black!70}] coordinates {(0,0) += (0,0.0362167) -= (0,0) (1,0) += (0,0.0362167) -= (0,0) (2,0) += (0,0.0362167) -= (0,0) (3,0) += (0,0.0362167) -= (0,0) (4,0) += (0,0.0362167) -= (0,0)};

\addplot[fill=cblue, draw=none, bar shift=0pt, eb, error bar style={color=black!70}] coordinates {(0,0) += (0,0.0362167) -= (0,0) (1,0) += (0,0.0362167) -= (0,0) (2,0) += (0,0.0362167) -= (0,0) (3,0) += (0,0.0362167) -= (0,0) (4,0) += (0,0.0362167) -= (0,0)};

\addplot[fill=caqua, draw=none, bar shift=5pt, eb, error bar style={color=black!70}] coordinates {(0,0) += (0,0.0362167) -= (0,0) (1,0) += (0,0.0362167) -= (0,0) (2,0) += (0,0.0362167) -= (0,0) (3,0) += (0,0.0362167) -= (0,0) (4,0) += (0,0.0362167) -= (0,0)};

\end{groupplot}\end{tikzpicture}
\end{adjustbox}
\caption{Directional versus radial decisions on the finite-state library, with no target labels and a declared unchanged channel ($\lambda$: concentrated shift; $\beta$: moderate shift). Intervals are 95\% exact binomial intervals.}\end{figure}

\begin{figure}[H]\centering
\begin{adjustbox}{max width=\linewidth}
\begin{tikzpicture}\begin{groupplot}[group style={group size=3 by 4, horizontal sep=0.8cm, vertical sep=1.2cm}, width=0.29\linewidth, height=3.8cm, xtick={0,2000,10000}, xticklabels={0,2k,10k}, x tick label style={font=\scriptsize,rotate=20,anchor=north east}, scaled x ticks=false, enlarge x limits=0.08]
\nextgroupplot[title={$\lambda{=}0.9$}, ymin=-0.05, ymax=1.05, ylabel={acceptance},xticklabels=\empty,
 legend style={figure legend,font=\scriptsize,at={(0,1.42)},anchor=south west,legend columns=4,
 /tikz/every even column/.append style={column sep=.55em}}]
\addplot[m_cert_radial, eb] coordinates {(-630,0) += (0,0.0362167) -= (0,0) (1370,0) += (0,0.0362167) -= (0,0) (9370,0) += (0,0.0362167) -= (0,0)};\addlegendentry{radial}

\addplot[m_cert_directional, eb] coordinates {(-810,1) += (0,0) -= (0,0.0362167) (1190,1) += (0,0) -= (0,0.0362167) (9190,1) += (0,0) -= (0,0.0362167)};\addlegendentry{directional}

\addplot[m_cert_joint, eb] coordinates {(-450,1) += (0,0) -= (0,0.0362167) (1550,1) += (0,0) -= (0,0.0362167) (9550,1) += (0,0) -= (0,0.0362167)};\addlegendentry{joint}

\addplot[m_direct_target_audit, eb] coordinates {(1730,1) += (0,0) -= (0,0.0362167) (9730,1) += (0,0) -= (0,0.0362167)};\addlegendentry{target audit}

\addplot[m_direct_conditional_model, eb] coordinates {(-90,1) += (0,0) -= (0,0.0362167) (1910,1) += (0,0) -= (0,0.0362167) (9910,1) += (0,0) -= (0,0.0362167)};\addlegendentry{conditional model}

\addplot[m_always_adapt, eb] coordinates {(90,1) += (0,0) -= (0,0.0362167) (2090,1) += (0,0) -= (0,0.0362167) (10090,1) += (0,0) -= (0,0.0362167)};\addlegendentry{always adapt}

\addplot[m_never_adapt, eb] coordinates {(270,0) += (0,0.0362167) -= (0,0) (2270,0) += (0,0.0362167) -= (0,0) (10270,0) += (0,0.0362167) -= (0,0)};\addlegendentry{retain baseline}

\nextgroupplot[title={$\lambda{=}0.5$}, ymin=-0.05, ymax=1.05,yticklabels=\empty,xticklabels=\empty]
\addplot[m_cert_radial, eb] coordinates {(-630,0) += (0,0.0362167) -= (0,0) (1370,0) += (0,0.0362167) -= (0,0) (9370,0) += (0,0.0362167) -= (0,0)};

\addplot[m_cert_directional, eb] coordinates {(-810,1) += (0,0) -= (0,0.0362167) (1190,1) += (0,0) -= (0,0.0362167) (9190,1) += (0,0) -= (0,0.0362167)};

\addplot[m_cert_joint, eb] coordinates {(-450,1) += (0,0) -= (0,0.0362167) (1550,1) += (0,0) -= (0,0.0362167) (9550,1) += (0,0) -= (0,0.0362167)};

\addplot[m_direct_target_audit, eb] coordinates {(1730,1) += (0,0) -= (0,0.0362167) (9730,1) += (0,0) -= (0,0.0362167)};

\addplot[m_direct_conditional_model, eb] coordinates {(-90,1) += (0,0) -= (0,0.0362167) (1910,1) += (0,0) -= (0,0.0362167) (9910,1) += (0,0) -= (0,0.0362167)};

\addplot[m_always_adapt, eb] coordinates {(90,1) += (0,0) -= (0,0.0362167) (2090,1) += (0,0) -= (0,0.0362167) (10090,1) += (0,0) -= (0,0.0362167)};

\addplot[m_never_adapt, eb] coordinates {(270,0) += (0,0.0362167) -= (0,0) (2270,0) += (0,0.0362167) -= (0,0) (10270,0) += (0,0.0362167) -= (0,0)};

\nextgroupplot[title={$\beta{=}2$}, ymin=-0.05, ymax=1.05,yticklabels=\empty,xticklabels=\empty]
\addplot[m_cert_radial, eb] coordinates {(-630,0) += (0,0.0362167) -= (0,0) (1370,0) += (0,0.0362167) -= (0,0) (9370,0) += (0,0.0362167) -= (0,0)};

\addplot[m_cert_directional, eb] coordinates {(-810,0) += (0,0.0362167) -= (0,0) (1190,0) += (0,0.0362167) -= (0,0) (9190,0) += (0,0.0362167) -= (0,0)};

\addplot[m_cert_joint, eb] coordinates {(-450,0) += (0,0.0362167) -= (0,0) (1550,0) += (0,0.0362167) -= (0,0) (9550,0) += (0,0.0362167) -= (0,0)};

\addplot[m_direct_target_audit, eb] coordinates {(1730,0.68) += (0,0.0897801) -= (0,0.100767) (9730,1) += (0,0) -= (0,0.0362167)};

\addplot[m_direct_conditional_model, eb] coordinates {(-90,1) += (0,0) -= (0,0.0362167) (1910,0.05) += (0,0.0628349) -= (0,0.0335681) (9910,1) += (0,0) -= (0,0.0362167)};

\addplot[m_always_adapt, eb] coordinates {(90,1) += (0,0) -= (0,0.0362167) (2090,1) += (0,0) -= (0,0.0362167) (10090,1) += (0,0) -= (0,0.0362167)};

\addplot[m_never_adapt, eb] coordinates {(270,0) += (0,0.0362167) -= (0,0) (2270,0) += (0,0.0362167) -= (0,0) (10270,0) += (0,0.0362167) -= (0,0)};

\nextgroupplot[title={$\beta{=}0$}, ymin=-0.05, ymax=1.05,xticklabels=\empty]
\addplot[m_cert_radial, eb] coordinates {(-630,0) += (0,0.0362167) -= (0,0) (1370,0) += (0,0.0362167) -= (0,0) (9370,0) += (0,0.0362167) -= (0,0)};

\addplot[m_cert_directional, eb] coordinates {(-810,0) += (0,0.0362167) -= (0,0) (1190,0) += (0,0.0362167) -= (0,0) (9190,0) += (0,0.0362167) -= (0,0)};

\addplot[m_cert_joint, eb] coordinates {(-450,0) += (0,0.0362167) -= (0,0) (1550,0) += (0,0.0362167) -= (0,0) (9550,0) += (0,0.0362167) -= (0,0)};

\addplot[m_direct_target_audit, eb] coordinates {(1730,0) += (0,0.0362167) -= (0,0) (9730,0) += (0,0.0362167) -= (0,0)};

\addplot[m_direct_conditional_model, eb] coordinates {(-90,0) += (0,0.0362167) -= (0,0) (1910,0) += (0,0.0362167) -= (0,0) (9910,0) += (0,0.0362167) -= (0,0)};

\addplot[m_always_adapt, eb] coordinates {(90,1) += (0,0) -= (0,0.0362167) (2090,1) += (0,0) -= (0,0.0362167) (10090,1) += (0,0) -= (0,0.0362167)};

\addplot[m_never_adapt, eb] coordinates {(270,0) += (0,0.0362167) -= (0,0) (2270,0) += (0,0.0362167) -= (0,0) (10270,0) += (0,0.0362167) -= (0,0)};

\nextgroupplot[title={$\beta{=}{-}2$}, ymin=-0.05, ymax=1.05,yticklabels=\empty,xticklabels=\empty]
\addplot[m_cert_radial, eb] coordinates {(-630,0) += (0,0.0362167) -= (0,0) (1370,0) += (0,0.0362167) -= (0,0) (9370,0) += (0,0.0362167) -= (0,0)};

\addplot[m_cert_directional, eb] coordinates {(-810,0) += (0,0.0362167) -= (0,0) (1190,0) += (0,0.0362167) -= (0,0) (9190,0) += (0,0.0362167) -= (0,0)};

\addplot[m_cert_joint, eb] coordinates {(-450,0) += (0,0.0362167) -= (0,0) (1550,0) += (0,0.0362167) -= (0,0) (9550,0) += (0,0.0362167) -= (0,0)};

\addplot[m_direct_target_audit, eb] coordinates {(1730,0) += (0,0.0362167) -= (0,0) (9730,0) += (0,0.0362167) -= (0,0)};

\addplot[m_direct_conditional_model, eb] coordinates {(-90,0) += (0,0.0362167) -= (0,0) (1910,0) += (0,0.0362167) -= (0,0) (9910,0) += (0,0.0362167) -= (0,0)};

\addplot[m_always_adapt, eb] coordinates {(90,1) += (0,0) -= (0,0.0362167) (2090,1) += (0,0) -= (0,0.0362167) (10090,1) += (0,0) -= (0,0.0362167)};

\addplot[m_never_adapt, eb] coordinates {(270,0) += (0,0.0362167) -= (0,0) (2270,0) += (0,0.0362167) -= (0,0) (10270,0) += (0,0.0362167) -= (0,0)};

\nextgroupplot[hide axis]

\nextgroupplot[title={$\lambda{=}0.9$},xlabel={target labels}, scaled y ticks=false, yticklabel style={/pgf/number format/fixed, /pgf/number format/precision=3}, ylabel={deployed excess risk}]
\addplot[m_cert_radial, eb] coordinates {(-630,0) += (0,0) -= (0,0) (1370,0) += (0,0) -= (0,0) (9370,0) += (0,0) -= (0,0)};

\addplot[m_cert_directional, eb] coordinates {(-810,-0.800226) += (0,1.11022e-16) -= (0,1.11022e-16) (1190,-0.800226) += (0,1.11022e-16) -= (0,1.11022e-16) (9190,-0.800226) += (0,1.11022e-16) -= (0,1.11022e-16)};

\addplot[m_cert_joint, eb] coordinates {(-450,-0.800226) += (0,1.11022e-16) -= (0,1.11022e-16) (1550,-0.800226) += (0,1.11022e-16) -= (0,1.11022e-16) (9550,-0.800226) += (0,1.11022e-16) -= (0,1.11022e-16)};

\addplot[m_direct_target_audit, eb] coordinates {(1730,-0.800226) += (0,1.11022e-16) -= (0,1.11022e-16) (9730,-0.800226) += (0,1.11022e-16) -= (0,1.11022e-16)};

\addplot[m_direct_conditional_model, eb] coordinates {(-90,-0.800226) += (0,1.11022e-16) -= (0,1.11022e-16) (1910,-0.800226) += (0,1.11022e-16) -= (0,1.11022e-16) (9910,-0.800226) += (0,1.11022e-16) -= (0,1.11022e-16)};

\addplot[m_always_adapt, eb] coordinates {(90,-0.800226) += (0,0) -= (0,0) (2090,-0.800226) += (0,0) -= (0,0) (10090,-0.800226) += (0,0) -= (0,0)};

\addplot[m_never_adapt, eb] coordinates {(270,0) += (0,0) -= (0,0) (2270,0) += (0,0) -= (0,0) (10270,0) += (0,0) -= (0,0)};

\nextgroupplot[title={$\lambda{=}0.5$},xlabel={target labels}, scaled y ticks=false, yticklabel style={/pgf/number format/fixed, /pgf/number format/precision=3}]
\addplot[m_cert_radial, eb] coordinates {(-630,0) += (0,0) -= (0,0) (1370,0) += (0,0) -= (0,0) (9370,0) += (0,0) -= (0,0)};

\addplot[m_cert_directional, eb] coordinates {(-810,-0.40113) += (0,0) -= (0,0) (1190,-0.40113) += (0,0) -= (0,0) (9190,-0.40113) += (0,0) -= (0,0)};

\addplot[m_cert_joint, eb] coordinates {(-450,-0.40113) += (0,0) -= (0,0) (1550,-0.40113) += (0,0) -= (0,0) (9550,-0.40113) += (0,0) -= (0,0)};

\addplot[m_direct_target_audit, eb] coordinates {(1730,-0.40113) += (0,0) -= (0,0) (9730,-0.40113) += (0,0) -= (0,0)};

\addplot[m_direct_conditional_model, eb] coordinates {(-90,-0.40113) += (0,0) -= (0,0) (1910,-0.40113) += (0,0) -= (0,0) (9910,-0.40113) += (0,0) -= (0,0)};

\addplot[m_always_adapt, eb] coordinates {(90,-0.40113) += (0,0) -= (0,0) (2090,-0.40113) += (0,0) -= (0,0) (10090,-0.40113) += (0,0) -= (0,0)};

\addplot[m_never_adapt, eb] coordinates {(270,0) += (0,0) -= (0,0) (2270,0) += (0,0) -= (0,0) (10270,0) += (0,0) -= (0,0)};

\nextgroupplot[title={$\beta{=}2$},xlabel={target labels}, scaled y ticks=false, yticklabel style={/pgf/number format/fixed, /pgf/number format/precision=3}]
\addplot[m_cert_radial, eb] coordinates {(-630,0) += (0,0) -= (0,0) (1370,0) += (0,0) -= (0,0) (9370,0) += (0,0) -= (0,0)};

\addplot[m_cert_directional, eb] coordinates {(-810,0) += (0,0) -= (0,0) (1190,0) += (0,0) -= (0,0) (9190,0) += (0,0) -= (0,0)};

\addplot[m_cert_joint, eb] coordinates {(-450,0) += (0,0) -= (0,0) (1550,0) += (0,0) -= (0,0) (9550,0) += (0,0) -= (0,0)};

\addplot[m_direct_target_audit, eb] coordinates {(1730,-0.0664627) += (0,0.00909223) -= (0,0.00909223) (9730,-0.0977393) += (0,0) -= (0,0)};

\addplot[m_direct_conditional_model, eb] coordinates {(-90,-0.0977393) += (0,0) -= (0,0) (1910,-0.00488697) += (0,0.00424803) -= (0,0.00424803) (9910,-0.0977393) += (0,0) -= (0,0)};

\addplot[m_always_adapt, eb] coordinates {(90,-0.0977393) += (0,0) -= (0,0) (2090,-0.0977393) += (0,0) -= (0,0) (10090,-0.0977393) += (0,0) -= (0,0)};

\addplot[m_never_adapt, eb] coordinates {(270,0) += (0,0) -= (0,0) (2270,0) += (0,0) -= (0,0) (10270,0) += (0,0) -= (0,0)};

\nextgroupplot[title={$\beta{=}0$},xlabel={target labels}, scaled y ticks=false, yticklabel style={/pgf/number format/fixed, /pgf/number format/precision=3}, ymin=-0.01, ymax=0.01]
\addplot[m_cert_radial, eb] coordinates {(-630,0) += (0,0) -= (0,0) (1370,0) += (0,0) -= (0,0) (9370,0) += (0,0) -= (0,0)};

\addplot[m_cert_directional, eb] coordinates {(-810,0) += (0,0) -= (0,0) (1190,0) += (0,0) -= (0,0) (9190,0) += (0,0) -= (0,0)};

\addplot[m_cert_joint, eb] coordinates {(-450,0) += (0,0) -= (0,0) (1550,0) += (0,0) -= (0,0) (9550,0) += (0,0) -= (0,0)};

\addplot[m_direct_target_audit, eb] coordinates {(1730,0) += (0,0) -= (0,0) (9730,0) += (0,0) -= (0,0)};

\addplot[m_direct_conditional_model, eb] coordinates {(-90,0) += (0,0) -= (0,0) (1910,0) += (0,0) -= (0,0) (9910,0) += (0,0) -= (0,0)};

\addplot[m_always_adapt, eb] coordinates {(90,1.27676e-17) += (0,2.75866e-18) -= (0,2.75866e-18) (2090,1.27676e-17) += (0,2.75866e-18) -= (0,2.75866e-18) (10090,1.27676e-17) += (0,2.75866e-18) -= (0,2.75866e-18)};

\addplot[m_never_adapt, eb] coordinates {(270,0) += (0,0) -= (0,0) (2270,0) += (0,0) -= (0,0) (10270,0) += (0,0) -= (0,0)};

\nextgroupplot[title={$\beta{=}{-}2$},xlabel={target labels}, scaled y ticks=false, yticklabel style={/pgf/number format/fixed, /pgf/number format/precision=3}, ymin=-0.01, ymax=0.01]
\addplot[m_cert_radial, eb] coordinates {(-630,0) += (0,0) -= (0,0) (1370,0) += (0,0) -= (0,0) (9370,0) += (0,0) -= (0,0)};

\addplot[m_cert_directional, eb] coordinates {(-810,0) += (0,0) -= (0,0) (1190,0) += (0,0) -= (0,0) (9190,0) += (0,0) -= (0,0)};

\addplot[m_cert_joint, eb] coordinates {(-450,0) += (0,0) -= (0,0) (1550,0) += (0,0) -= (0,0) (9550,0) += (0,0) -= (0,0)};

\addplot[m_direct_target_audit, eb] coordinates {(1730,0) += (0,0) -= (0,0) (9730,0) += (0,0) -= (0,0)};

\addplot[m_direct_conditional_model, eb] coordinates {(-90,0) += (0,0) -= (0,0) (1910,0) += (0,0) -= (0,0) (9910,0) += (0,0) -= (0,0)};

\addplot[m_always_adapt, eb] coordinates {(90,0) += (0,0) -= (0,0) (2090,0) += (0,0) -= (0,0) (10090,0) += (0,0) -= (0,0)};

\addplot[m_never_adapt, eb] coordinates {(270,0) += (0,0) -= (0,0) (2270,0) += (0,0) -= (0,0) (10270,0) += (0,0) -= (0,0)};

\nextgroupplot[hide axis]

\end{groupplot}\end{tikzpicture}
\end{adjustbox}
\caption{Experiment 2: acceptance and deployed excess risk against target-label budget ($B>0$ uses the audited channel).}\label{fig:e2}\end{figure}

\section{Experiment 3: conditional shift and localized audits}
Fixed $h_0=0$, $h_1=\mathbb{1}\{i<d/4\}$, disagreement region $G$ = first quarter of states. Input marginals are equal so $\MMD=0$ exactly; all change in target risk comes from the label channel. Source audit 100{,}000, 100 repetitions.
\begin{table}[H]\centering\small
\caption{Conditional penalty and acceptance at 2{,}000 target labels. $\mathcal{E}_t$: true target excess of $h_1$. Known: localized plug-in with the exact channel.}
\begin{adjustbox}{max width=\linewidth}\begin{tabular}{l r r r rrr rrr}\toprule
Shift & $\zeta$ & $\mathcal{E}_t$ & true $C$ & known & global & localized & \multicolumn{3}{c}{acceptance (\%)}\\
& & & & \multicolumn{3}{c}{certified $C$} & global & localized & direct audit\\\midrule
inside $G$ & 0.3 & +0.164 & -0.246 & -0.233 & +0.465 & -0.157 & 0 & 0 & 0 \\
inside $G$ & 1.0 & -0.410 & -0.819 & -0.780 & +1.014 & -0.729 & 0 & 100 & 100 \\
outside $G$ & 0.3 & +0.410 & +0.000 & +0.000 & +0.552 & +0.035 & 0 & 0 & 0 \\
outside $G$ & 1.0 & +0.410 & +0.000 & +0.000 & +1.180 & +0.035 & 0 & 0 & 0 \\
everywhere & 0.3 & +0.164 & -0.246 & -0.234 & +0.844 & -0.155 & 0 & 0 & 0 \\
everywhere & 1.0 & -0.410 & -0.819 & -0.778 & +2.020 & -0.727 & 0 & 100 & 100 \\
\bottomrule\end{tabular}\end{adjustbox}\end{table}
\begin{table}[H]\centering\small
\caption{Rare disagreement regions ($\zeta=1$ inside $G$). Half-widths are the mean statewise audit half-width in $G$ for a uniform versus a targeted (label only draws in $G$) target audit; ``draws'' is the number of unlabeled target inputs the targeted audit had to screen.}
\begin{adjustbox}{max width=\linewidth}\begin{tabular}{rr r rr r rrr rr}\toprule
$p_t(G)$ & labels & $\mathcal{E}_t$ & \multicolumn{2}{c}{half-width} & draws & \multicolumn{3}{c}{mean bound} & \multicolumn{2}{c}{acc.\ (\%)}\\\cmidrule(lr){4-5}\cmidrule(lr){7-9}\cmidrule(lr){10-11}
& & & uniform & targeted & & global & localized & targeted & targeted & audit\\\midrule
0.01 & 2,000 & -0.009 & 0.581 & 0.032 & 200,020 & +0.200 & +0.022 & +0.015 & 0 & 2 \\
0.01 & 20,000 & -0.009 & 0.132 & 0.009 & 1,998,783 & +0.108 & +0.016 & +0.015 & 0 & 100 \\
0.05 & 2,000 & -0.045 & 0.211 & 0.033 & 40,015 & +0.335 & +0.023 & +0.008 & 2 & 100 \\
0.05 & 20,000 & -0.045 & 0.049 & 0.009 & 400,204 & +0.245 & +0.009 & +0.006 & 2 & 100 \\
0.2 & 2,000 & -0.180 & 0.085 & 0.032 & 10,013 & +0.789 & -0.055 & -0.074 & 100 & 100 \\
0.2 & 20,000 & -0.180 & 0.022 & 0.009 & 100,013 & +0.694 & -0.078 & -0.082 & 100 & 100 \\
\bottomrule\end{tabular}\end{adjustbox}\end{table}
The invalid control that declares the channel unchanged violated its bound in 0 of 5400 evaluations; it never accepted, so the violations did not cause harm in this design.
\begin{figure}[H]\centering
\begin{adjustbox}{max width=\linewidth}
\begin{tikzpicture}\begin{groupplot}[group style={group size=3 by 2, horizontal sep=0.7cm, vertical sep=3.8cm}, width=0.305\linewidth, height=5.6cm]
\nextgroupplot[title={channel shift inside $G$}, xlabel={$\zeta$}, ylabel={certified $C$ ($B=2{,}000$)},
 legend style={figure legend,font=\scriptsize,at={(0,1.38)},anchor=south west,legend columns=2,
 /tikz/every even column/.append style={column sep=.7em}}]
\addplot[color=corange, mark=*, line width=0.9pt, mark size=1.5pt, eb] coordinates {(0,0.172026) += (0,0.00297271) -= (0,0.00297271) (0.1,0.27472) += (0,0.0032962) -= (0,0.0032962) (0.3,0.464697) += (0,0.0036307) -= (0,0.0036307) (0.6,0.720714) += (0,0.00363849) -= (0,0.00363849) (1,1.01351) += (0,0.00276122) -= (0,0.00276122)};\addlegendentry{global penalty (audited)}

\addplot[color=cblue, mark=*, line width=0.9pt, mark size=1.5pt, eb] coordinates {(0,0.0347018) += (0,0.000887052) -= (0,0.000887052) (0.1,-0.023795) += (0,0.00148314) -= (0,0.00148314) (0.3,-0.156843) += (0,0.00221971) -= (0,0.00221971) (0.6,-0.37935) += (0,0.00289331) -= (0,0.00289331) (1,-0.728949) += (0,0.00196565) -= (0,0.00196565)};\addlegendentry{localized penalty (audited)}

\addplot[color=caqua, mark=*, line width=0.9pt, mark size=1.5pt, eb] coordinates {(0,0) += (0,0) -= (0,0) (0.1,-0.0778091) += (0,0.000120698) -= (0,0.000120698) (0.3,-0.233398) += (0,0.000380392) -= (0,0.000380392) (0.6,-0.467611) += (0,0.000747978) -= (0,0.000747978) (1,-0.780284) += (0,0.00122417) -= (0,0.00122417)};\addlegendentry{localized penalty (known channel)}

\addplot[black, dashed, line width=0.8pt] coordinates {(0,0) (0.1,-0.0819098) (0.3,-0.245729) (0.6,-0.491459) (1,-0.819098)};\addlegendentry{true conditional term}

\nextgroupplot[title={channel shift outside $G$}, xlabel={$\zeta$}]
\addplot[color=corange, mark=*, line width=0.9pt, mark size=1.5pt, eb] coordinates {(0,0.170672) += (0,0.002934) -= (0,0.002934) (0.1,0.312101) += (0,0.00321662) -= (0,0.00321662) (0.3,0.552117) += (0,0.00361805) -= (0,0.00361805) (0.6,0.865075) += (0,0.00371145) -= (0,0.00371145) (1,1.17996) += (0,0.00214949) -= (0,0.00214949)};

\addplot[color=cblue, mark=*, line width=0.9pt, mark size=1.5pt, eb] coordinates {(0,0.0343471) += (0,0.000907563) -= (0,0.000907563) (0.1,0.0349239) += (0,0.000940111) -= (0,0.000940111) (0.3,0.0349931) += (0,0.000723726) -= (0,0.000723726) (0.6,0.0352559) += (0,0.000968353) -= (0,0.000968353) (1,0.0350337) += (0,0.000886831) -= (0,0.000886831)};

\addplot[color=caqua, mark=*, line width=0.9pt, mark size=1.5pt, eb] coordinates {(0,0) += (0,0) -= (0,0) (0.1,0) += (0,0) -= (0,0) (0.3,0) += (0,0) -= (0,0) (0.6,0) += (0,0) -= (0,0) (1,0) += (0,0) -= (0,0)};

\addplot[black, dashed, line width=0.8pt] coordinates {(0,0) (0.1,0) (0.3,0) (0.6,0) (1,0)};

\nextgroupplot[title={channel shift everywhere}, xlabel={$\zeta$}]
\addplot[color=corange, mark=*, line width=0.9pt, mark size=1.5pt, eb] coordinates {(0,0.169958) += (0,0.00268207) -= (0,0.00268207) (0.1,0.419061) += (0,0.00355076) -= (0,0.00355076) (0.3,0.844312) += (0,0.00467663) -= (0,0.00467663) (0.6,1.41448) += (0,0.00444612) -= (0,0.00444612) (1,2.01963) += (0,0.00115289) -= (0,0.00115289)};

\addplot[color=cblue, mark=*, line width=0.9pt, mark size=1.5pt, eb] coordinates {(0,0.0350741) += (0,0.000963855) -= (0,0.000963855) (0.1,-0.0233521) += (0,0.00161479) -= (0,0.00161479) (0.3,-0.155209) += (0,0.00252884) -= (0,0.00252884) (0.6,-0.380784) += (0,0.00318921) -= (0,0.00318921) (1,-0.727462) += (0,0.00215308) -= (0,0.00215308)};

\addplot[color=caqua, mark=*, line width=0.9pt, mark size=1.5pt, eb] coordinates {(0,0) += (0,0) -= (0,0) (0.1,-0.0779149) += (0,0.000108753) -= (0,0.000108753) (0.3,-0.233555) += (0,0.000382071) -= (0,0.000382071) (0.6,-0.466978) += (0,0.000775121) -= (0,0.000775121) (1,-0.778019) += (0,0.00126437) -= (0,0.00126437)};

\addplot[black, dashed, line width=0.8pt] coordinates {(0,0) (0.1,-0.0819098) (0.3,-0.245729) (0.6,-0.491459) (1,-0.819098)};

\nextgroupplot[title={acceptance, shift inside $G$ ($B=2{,}000$)}, xlabel={$\zeta$}, ylabel={acceptance}, ymin=-0.05, ymax=1.05,
 legend style={figure legend,font=\scriptsize,at={(0,1.42)},anchor=south west,legend columns=3,
 /tikz/every even column/.append style={column sep=.6em}}]
\addplot[color=corange, mark=*, line width=0.9pt, mark size=1.5pt, eb] coordinates {(0,0) += (0,0.0362167) -= (0,0) (0.1,0) += (0,0.0362167) -= (0,0) (0.3,0) += (0,0.0362167) -= (0,0) (0.6,0) += (0,0.0362167) -= (0,0) (1,0) += (0,0.0362167) -= (0,0)};\addlegendentry{global (audited)}

\addplot[color=cblue, mark=*, line width=0.9pt, mark size=1.5pt, eb] coordinates {(0,0) += (0,0.0362167) -= (0,0) (0.1,0) += (0,0.0362167) -= (0,0) (0.3,0) += (0,0.0362167) -= (0,0) (0.6,0) += (0,0.0362167) -= (0,0) (1,1) += (0,0) -= (0,0.0362167)};\addlegendentry{localized (audited)}

\addplot[color=cviolet, mark=*, line width=0.9pt, mark size=1.5pt, eb] coordinates {(0,0) += (0,0.0362167) -= (0,0) (0.1,0) += (0,0.0362167) -= (0,0) (0.3,0) += (0,0.0362167) -= (0,0) (0.6,0) += (0,0.0362167) -= (0,0) (1,1) += (0,0) -= (0,0.0362167)};\addlegendentry{localized, targeted audit}

\addplot[color=caqua, mark=*, line width=0.9pt, mark size=1.5pt, eb] coordinates {(0,0) += (0,0.0362167) -= (0,0) (0.1,0) += (0,0.0362167) -= (0,0) (0.3,0) += (0,0.0362167) -= (0,0) (0.6,0) += (0,0.0362167) -= (0,0) (1,1) += (0,0) -= (0,0.0362167)};\addlegendentry{localized (known channel)}

\addplot[color=cyellow, mark=*, line width=0.9pt, mark size=1.5pt, eb] coordinates {(0,0) += (0,0.0362167) -= (0,0) (0.1,0) += (0,0.0362167) -= (0,0) (0.3,0) += (0,0.0362167) -= (0,0) (0.6,0.45) += (0,0.10272) -= (0,0.0996798) (1,1) += (0,0) -= (0,0.0362167)};\addlegendentry{direct labeled-target audit}

\nextgroupplot[title={audit half-width in $G$ ($\zeta=1$)}, xmode=log, xlabel={target mass on $G$}, ylabel={half-width},
 legend style={figure legend,at={(0.98,0.98)},anchor=north east,font=\tiny}]
\addplot[color=cblue, mark=*, line width=0.9pt, mark size=1.5pt] coordinates {(0.01,0.581238) (0.05,0.211087) (0.2,0.0845811)};\addlegendentry{uniform, $B=2{,}000$}

\addplot[color=cblue, mark=*, line width=0.9pt, mark size=1.5pt,densely dotted] coordinates {(0.01,0.132308) (0.05,0.048527) (0.2,0.0216403)};\addlegendentry{uniform, $B=20{,}000$}

\addplot[color=cviolet, mark=*, line width=0.9pt, mark size=1.5pt] coordinates {(0.01,0.0321398) (0.05,0.032513) (0.2,0.0320445)};\addlegendentry{targeted, $B=2{,}000$}

\addplot[color=cviolet, mark=*, line width=0.9pt, mark size=1.5pt,densely dotted] coordinates {(0.01,0.00914139) (0.05,0.009152) (0.2,0.0091473)};\addlegendentry{targeted, $B=20{,}000$}

\nextgroupplot[title={certified bound, rare $G$ ($\zeta=1$, $B=2{,}000$)}, xmode=log, xlabel={target mass on $G$}, enlarge y limits={upper, value=0.45},
 legend style={figure legend,at={(0.02,0.98)},anchor=north west,font=\tiny}]
\addplot[color=corange, mark=*, line width=0.9pt, mark size=1.5pt, eb] coordinates {(0.01,0.199842) += (0,0.00259444) -= (0,0.00259444) (0.05,0.335363) += (0,0.00244531) -= (0,0.00244531) (0.2,0.789498) += (0,0.00312376) -= (0,0.00312376)};\addlegendentry{global}

\addplot[color=cblue, mark=*, line width=0.9pt, mark size=1.5pt, eb] coordinates {(0.01,0.0221906) += (0,0.000350226) -= (0,0.000350226) (0.05,0.0228199) += (0,0.000911256) -= (0,0.000911256) (0.2,-0.0546126) += (0,0.00169139) -= (0,0.00169139)};\addlegendentry{localized}

\addplot[color=cviolet, mark=*, line width=0.9pt, mark size=1.5pt, eb] coordinates {(0.01,0.0151875) += (0,0.00025689) -= (0,0.00025689) (0.05,0.00792073) += (0,0.000641586) -= (0,0.000641586) (0.2,-0.0739882) += (0,0.00108182) -= (0,0.00108182)};\addlegendentry{localized, targeted}

\addplot[black, dashed, line width=0.8pt] coordinates {(0.01,-0.009) (0.05,-0.045) (0.2,-0.18)};\addlegendentry{true target excess}

\end{groupplot}\end{tikzpicture}
\end{adjustbox}
\caption{Conditional shift changes target risk while the marginal MMD is exactly zero; localization confines the penalty and audit to the disagreement region $G$. Intervals are 95\% over repetitions.}\end{figure}

\clearpage
\section{Experiment 4: Covertype}
Cover types 1 vs 2, first 10 cartographic features, standardized and quantized by pilot-fitted $k$-means; classifiers see only the cluster centroid. Case A reweights the frozen population by elevation rank (channel unchanged by construction, declared). Case B is a real geographic shift (both channels audited when target labels exist). 5 split seeds $\times$ 10 repetitions; $n=m=r=10{,}000$, source audit 50{,}000, $u=20{,}000$. Exact population risks are evaluated after the decision.

Because classifiers are constant within states and Case A keeps the label channel fixed, the source-optimal per-state rule is also target-optimal. With the prespecified base library the pilot-selected $h_0$ is already close to that rule, so almost no useful update exists. The \emph{rich} variant (trees to depth 10 with class-weight variants, deployed baseline restricted to constants, logistic models and depth-1 trees) is an exploratory stress test added after observing this low-availability regime; it is reported separately and is not used as confirmatory evidence of power.

\begin{table}[H]\centering\small
\caption{Covertype results; risks in percentage points. ``Useful'': \% of splits with a library member of negative target excess. ``Always'': the unconditional adaptation comparator. Certified bounds are means; 0 labels uses the declared channel, 10k labels the audited channel.}
\begin{adjustbox}{max width=\linewidth}\begin{tabular}{l rr rr rrr r}\toprule
Target & Useful & Best & \multicolumn{2}{c}{Always adapt} & \multicolumn{3}{c}{mean bound (pp)} & Audit 10k\\\cmidrule(lr){4-5}\cmidrule(lr){6-8}
& (\%) & (pp) & harmful (\%) & excess (pp) & radial, 0 & dir., 0 & dir., 10k & acc.\ (\%)\\\midrule
\multicolumn{9}{l}{\textit{8 states, base library; library size $L\in\{4,5,6\}$}} \\
A: $\beta=0$ & 0 & +0.00 & 20 & +0.02 & +1.0 & +0.6 & +1.1 & 0 \\
A: $\beta=2$ & 20 & -0.03 & 6 & -0.00 & +1.5 & +0.8 & +1.7 & 0 \\
B: area 1$\to$3 & 60 & -0.77 & 20 & +0.55 & +2.2 & +0.7 & +2.6 & 0 \\
B: area 1$\to$4 & 80 & -0.51 & 0 & +0.00 & +2.6 & +0.8 & +0.9 & 80 \\
\midrule
\multicolumn{9}{l}{\textit{32 states, base library; library size $L\in\{6,7,8\}$}} \\
A: $\beta=0$ & 20 & -0.00 & 18 & +0.02 & +1.8 & +0.8 & +1.7 & 0 \\
A: $\beta=2$ & 40 & -0.00 & 8 & +0.04 & +2.6 & +0.9 & +2.2 & 0 \\
B: area 1$\to$3 & 40 & -0.35 & 100 & +2.75 & +24.4 & +3.9 & +13.2 & 0 \\
B: area 1$\to$4 & 80 & -0.26 & 60 & +0.89 & +28.7 & +3.4 & +4.0 & 40 \\
\midrule
\multicolumn{9}{l}{\textit{32 states, rich library, simple deployed baseline; library size $L\in\{8,10,15\}$}} \\
A: $\beta=0$ & 80 & -0.20 & 8 & -0.17 & +5.7 & +1.6 & +3.9 & 0 \\
A: $\beta=2$ & 80 & -0.32 & 10 & -0.25 & +8.4 & +2.2 & +5.5 & 0 \\
B: area 1$\to$3 & 60 & -0.54 & 100 & +2.88 & +28.1 & +4.1 & +16.2 & 2 \\
B: area 1$\to$4 & 80 & -0.23 & 100 & +5.66 & +42.4 & +5.1 & +8.3 & 40 \\
\bottomrule\end{tabular}\end{adjustbox}\end{table}
The certificate did exactly what an accept-or-retain mechanism is designed to do in the geographic-shift panels: it blocked every update when the certified evidence was insufficient. In the prespecified 32-state base library, unconditional adaptation is harmful in all area-$1\to3$ repetitions and raises mean risk by $2.75$ percentage points; certification retains the baseline and therefore incurs zero deployed excess. The exploratory rich stress test raises the unconditional harm to $5.66$ points in area $1\to4$, and the certificate again blocks it. No certificate or direct audit makes a harmful acceptance. The price of this risk aversion is low power: no certificate variant accepts a beneficial Covertype update, while the direct labeled audit with 10{,}000 target labels captures some small gains.

\begin{table}[H]\centering\small
\caption{Complete real-data ledger: 32 states, rich library, area $1\to4$, 10{,}000 target labels, directional certificate with audited channel, for the candidate the certificate selected (smallest bound). Percentage points, mean over 50 repetitions. The selected candidate is itself harmful (true excess positive) and is correctly rejected.}
\begin{adjustbox}{max width=\linewidth}\begin{tabular}{lr}\toprule
Component & value\\\midrule
$\hat c_s$ & +1.49 \\
$S$ & +3.83 \\
$\langle\hat p_t-\hat p_s,\hat w\rangle$ & -1.70 \\
concentration slack & +0.99 \\
residual transfer ($\rho$) & +2.15 \\
$T$ (directional) & +1.44 \\
$C$ (audited conditional) & +3.01 \\
certificate & +8.28 \\
true source excess & +1.50 \\
true transfer term & -1.70 \\
true conditional term & +1.95 \\
true target excess & +1.76 \\
\bottomrule\end{tabular}\end{adjustbox}\end{table}
\begin{figure}[H]\centering
\begin{adjustbox}{max width=\linewidth}
\begin{tikzpicture}\begin{groupplot}[group style={group size=2 by 4, horizontal sep=1.2cm, vertical sep=1.1cm}, width=0.43\linewidth, height=4.2cm, xtick={0,2000,10000}, xticklabels={0,2k,10k}, x tick label style={font=\scriptsize,rotate=20,anchor=north east}, scaled x ticks=false, enlarge x limits=0.08]
\nextgroupplot[title={A: $\beta=0$}, ymin=-0.05, ymax=1.05, ylabel={acceptance},xticklabels=\empty,
 legend style={figure legend,font=\scriptsize,at={(0,1.45)},anchor=south west,legend columns=4,
 /tikz/every even column/.append style={column sep=.55em}}]
\addplot[m_cert_radial, eb] coordinates {(-630,0) += (0,0.0711217) -= (0,0) (1370,0) += (0,0.0711217) -= (0,0) (9370,0) += (0,0.0711217) -= (0,0)};\addlegendentry{radial}

\addplot[m_cert_directional, eb] coordinates {(-810,0) += (0,0.0711217) -= (0,0) (1190,0) += (0,0.0711217) -= (0,0) (9190,0) += (0,0.0711217) -= (0,0)};\addlegendentry{directional}

\addplot[m_cert_joint, eb] coordinates {(-450,0) += (0,0.0711217) -= (0,0) (1550,0) += (0,0.0711217) -= (0,0) (9550,0) += (0,0.0711217) -= (0,0)};\addlegendentry{joint}

\addplot[m_direct_target_audit, eb] coordinates {(1730,0) += (0,0.0711217) -= (0,0) (9730,0) += (0,0.0711217) -= (0,0)};\addlegendentry{target audit}

\addplot[m_direct_conditional_model, eb] coordinates {(-90,0) += (0,0.0711217) -= (0,0) (1910,0) += (0,0.0711217) -= (0,0) (9910,0) += (0,0.0711217) -= (0,0)};\addlegendentry{conditional model}

\addplot[m_always_adapt, eb] coordinates {(90,1) += (0,0) -= (0,0.0711217) (2090,1) += (0,0) -= (0,0.0711217) (10090,1) += (0,0) -= (0,0.0711217)};\addlegendentry{always adapt}

\addplot[m_never_adapt, eb] coordinates {(270,0) += (0,0.0711217) -= (0,0) (2270,0) += (0,0.0711217) -= (0,0) (10270,0) += (0,0.0711217) -= (0,0)};\addlegendentry{retain baseline}

\nextgroupplot[title={A: $\beta=2$}, ymin=-0.05, ymax=1.05,yticklabels=\empty,xticklabels=\empty]
\addplot[m_cert_radial, eb] coordinates {(-630,0) += (0,0.0711217) -= (0,0) (1370,0) += (0,0.0711217) -= (0,0) (9370,0) += (0,0.0711217) -= (0,0)};

\addplot[m_cert_directional, eb] coordinates {(-810,0) += (0,0.0711217) -= (0,0) (1190,0) += (0,0.0711217) -= (0,0) (9190,0) += (0,0.0711217) -= (0,0)};

\addplot[m_cert_joint, eb] coordinates {(-450,0) += (0,0.0711217) -= (0,0) (1550,0) += (0,0.0711217) -= (0,0) (9550,0) += (0,0.0711217) -= (0,0)};

\addplot[m_direct_target_audit, eb] coordinates {(1730,0) += (0,0.0711217) -= (0,0) (9730,0) += (0,0.0711217) -= (0,0)};

\addplot[m_direct_conditional_model, eb] coordinates {(-90,0) += (0,0.0711217) -= (0,0) (1910,0) += (0,0.0711217) -= (0,0) (9910,0) += (0,0.0711217) -= (0,0)};

\addplot[m_always_adapt, eb] coordinates {(90,1) += (0,0) -= (0,0.0711217) (2090,1) += (0,0) -= (0,0.0711217) (10090,1) += (0,0) -= (0,0.0711217)};

\addplot[m_never_adapt, eb] coordinates {(270,0) += (0,0.0711217) -= (0,0) (2270,0) += (0,0.0711217) -= (0,0) (10270,0) += (0,0.0711217) -= (0,0)};

\nextgroupplot[title={B: area $1\to3$}, ymin=-0.05, ymax=1.05,xticklabels=\empty]
\addplot[m_cert_radial, eb] coordinates {(-630,0) += (0,0.0711217) -= (0,0) (1370,0) += (0,0.0711217) -= (0,0) (9370,0) += (0,0.0711217) -= (0,0)};

\addplot[m_cert_directional, eb] coordinates {(-810,0) += (0,0.0711217) -= (0,0) (1190,0) += (0,0.0711217) -= (0,0) (9190,0) += (0,0.0711217) -= (0,0)};

\addplot[m_cert_joint, eb] coordinates {(-450,0) += (0,0.0711217) -= (0,0) (1550,0) += (0,0.0711217) -= (0,0) (9550,0) += (0,0.0711217) -= (0,0)};

\addplot[m_direct_target_audit, eb] coordinates {(1730,0) += (0,0.0711217) -= (0,0) (9730,0.02) += (0,0.0864695) -= (0,0.0194938)};

\addplot[m_direct_conditional_model, eb] coordinates {(-90,0) += (0,0.0711217) -= (0,0) (1910,0) += (0,0.0711217) -= (0,0) (9910,0) += (0,0.0711217) -= (0,0)};

\addplot[m_always_adapt, eb] coordinates {(90,1) += (0,0) -= (0,0.0711217) (2090,1) += (0,0) -= (0,0.0711217) (10090,1) += (0,0) -= (0,0.0711217)};

\addplot[m_never_adapt, eb] coordinates {(270,0) += (0,0.0711217) -= (0,0) (2270,0) += (0,0.0711217) -= (0,0) (10270,0) += (0,0.0711217) -= (0,0)};

\nextgroupplot[title={B: area $1\to4$}, ymin=-0.05, ymax=1.05,yticklabels=\empty,xticklabels=\empty]
\addplot[m_cert_radial, eb] coordinates {(-630,0) += (0,0.0711217) -= (0,0) (1370,0) += (0,0.0711217) -= (0,0) (9370,0) += (0,0.0711217) -= (0,0)};

\addplot[m_cert_directional, eb] coordinates {(-810,0) += (0,0.0711217) -= (0,0) (1190,0) += (0,0.0711217) -= (0,0) (9190,0) += (0,0.0711217) -= (0,0)};

\addplot[m_cert_joint, eb] coordinates {(-450,0) += (0,0.0711217) -= (0,0) (1550,0) += (0,0.0711217) -= (0,0) (9550,0) += (0,0.0711217) -= (0,0)};

\addplot[m_direct_target_audit, eb] coordinates {(1730,0) += (0,0.0711217) -= (0,0) (9730,0.4) += (0,0.148206) -= (0,0.135922)};

\addplot[m_direct_conditional_model, eb] coordinates {(-90,0) += (0,0.0711217) -= (0,0) (1910,0) += (0,0.0711217) -= (0,0) (9910,0) += (0,0.0711217) -= (0,0)};

\addplot[m_always_adapt, eb] coordinates {(90,1) += (0,0) -= (0,0.0711217) (2090,1) += (0,0) -= (0,0.0711217) (10090,1) += (0,0) -= (0,0.0711217)};

\addplot[m_never_adapt, eb] coordinates {(270,0) += (0,0.0711217) -= (0,0) (2270,0) += (0,0.0711217) -= (0,0) (10270,0) += (0,0.0711217) -= (0,0)};

\nextgroupplot[title={A: $\beta=0$},xlabel={target labels}, scaled y ticks=false, yticklabel style={/pgf/number format/fixed, /pgf/number format/precision=3}, ylabel={deployed excess (pp)}]
\addplot[m_cert_radial, eb] coordinates {(-630,0) += (0,0) -= (0,0) (1370,0) += (0,0) -= (0,0) (9370,0) += (0,0) -= (0,0)};

\addplot[m_cert_directional, eb] coordinates {(-810,0) += (0,0) -= (0,0) (1190,0) += (0,0) -= (0,0) (9190,0) += (0,0) -= (0,0)};

\addplot[m_cert_joint, eb] coordinates {(-450,0) += (0,0) -= (0,0) (1550,0) += (0,0) -= (0,0) (9550,0) += (0,0) -= (0,0)};

\addplot[m_direct_target_audit, eb] coordinates {(1730,0) += (0,0) -= (0,0) (9730,0) += (0,0) -= (0,0)};

\addplot[m_direct_conditional_model, eb] coordinates {(-90,0) += (0,0) -= (0,0) (1910,0) += (0,0) -= (0,0) (9910,0) += (0,0) -= (0,0)};

\addplot[m_always_adapt, eb] coordinates {(90,-0.174294) += (0,0.0408374) -= (0,0.0408374) (2090,-0.174294) += (0,0.0408374) -= (0,0.0408374) (10090,-0.174294) += (0,0.0408374) -= (0,0.0408374)};

\addplot[m_never_adapt, eb] coordinates {(270,0) += (0,0) -= (0,0) (2270,0) += (0,0) -= (0,0) (10270,0) += (0,0) -= (0,0)};

\nextgroupplot[title={A: $\beta=2$},xlabel={target labels}, scaled y ticks=false, yticklabel style={/pgf/number format/fixed, /pgf/number format/precision=3}]
\addplot[m_cert_radial, eb] coordinates {(-630,0) += (0,0) -= (0,0) (1370,0) += (0,0) -= (0,0) (9370,0) += (0,0) -= (0,0)};

\addplot[m_cert_directional, eb] coordinates {(-810,0) += (0,0) -= (0,0) (1190,0) += (0,0) -= (0,0) (9190,0) += (0,0) -= (0,0)};

\addplot[m_cert_joint, eb] coordinates {(-450,0) += (0,0) -= (0,0) (1550,0) += (0,0) -= (0,0) (9550,0) += (0,0) -= (0,0)};

\addplot[m_direct_target_audit, eb] coordinates {(1730,0) += (0,0) -= (0,0) (9730,0) += (0,0) -= (0,0)};

\addplot[m_direct_conditional_model, eb] coordinates {(-90,0) += (0,0) -= (0,0) (1910,0) += (0,0) -= (0,0) (9910,0) += (0,0) -= (0,0)};

\addplot[m_always_adapt, eb] coordinates {(90,-0.251488) += (0,0.0666353) -= (0,0.0666353) (2090,-0.251488) += (0,0.0666353) -= (0,0.0666353) (10090,-0.251488) += (0,0.0666353) -= (0,0.0666353)};

\addplot[m_never_adapt, eb] coordinates {(270,0) += (0,0) -= (0,0) (2270,0) += (0,0) -= (0,0) (10270,0) += (0,0) -= (0,0)};

\nextgroupplot[title={B: area $1\to3$},xlabel={target labels}, scaled y ticks=false, yticklabel style={/pgf/number format/fixed, /pgf/number format/precision=3}]
\addplot[m_cert_radial, eb] coordinates {(-630,0) += (0,0) -= (0,0) (1370,0) += (0,0) -= (0,0) (9370,0) += (0,0) -= (0,0)};

\addplot[m_cert_directional, eb] coordinates {(-810,0) += (0,0) -= (0,0) (1190,0) += (0,0) -= (0,0) (9190,0) += (0,0) -= (0,0)};

\addplot[m_cert_joint, eb] coordinates {(-450,0) += (0,0) -= (0,0) (1550,0) += (0,0) -= (0,0) (9550,0) += (0,0) -= (0,0)};

\addplot[m_direct_target_audit, eb] coordinates {(1730,0) += (0,0) -= (0,0) (9730,-0.0249425) += (0,0.0501238) -= (0,0.0501238)};

\addplot[m_direct_conditional_model, eb] coordinates {(-90,0) += (0,0) -= (0,0) (1910,0) += (0,0) -= (0,0) (9910,0) += (0,0) -= (0,0)};

\addplot[m_always_adapt, eb] coordinates {(90,2.88255) += (0,0.28437) -= (0,0.28437) (2090,2.88255) += (0,0.28437) -= (0,0.28437) (10090,2.88255) += (0,0.28437) -= (0,0.28437)};

\addplot[m_never_adapt, eb] coordinates {(270,0) += (0,0) -= (0,0) (2270,0) += (0,0) -= (0,0) (10270,0) += (0,0) -= (0,0)};

\nextgroupplot[title={B: area $1\to4$},xlabel={target labels}, scaled y ticks=false, yticklabel style={/pgf/number format/fixed, /pgf/number format/precision=3}]
\addplot[m_cert_radial, eb] coordinates {(-630,0) += (0,0) -= (0,0) (1370,0) += (0,0) -= (0,0) (9370,0) += (0,0) -= (0,0)};

\addplot[m_cert_directional, eb] coordinates {(-810,0) += (0,0) -= (0,0) (1190,0) += (0,0) -= (0,0) (9190,0) += (0,0) -= (0,0)};

\addplot[m_cert_joint, eb] coordinates {(-450,0) += (0,0) -= (0,0) (1550,0) += (0,0) -= (0,0) (9550,0) += (0,0) -= (0,0)};

\addplot[m_direct_target_audit, eb] coordinates {(1730,0) += (0,0) -= (0,0) (9730,-0.213648) += (0,0.0751216) -= (0,0.0751216)};

\addplot[m_direct_conditional_model, eb] coordinates {(-90,0) += (0,0) -= (0,0) (1910,0) += (0,0) -= (0,0) (9910,0) += (0,0) -= (0,0)};

\addplot[m_always_adapt, eb] coordinates {(90,5.66077) += (0,1.00483) -= (0,1.00483) (2090,5.66077) += (0,1.00483) -= (0,1.00483) (10090,5.66077) += (0,1.00483) -= (0,1.00483)};

\addplot[m_never_adapt, eb] coordinates {(270,0) += (0,0) -= (0,0) (2270,0) += (0,0) -= (0,0) (10270,0) += (0,0) -= (0,0)};

\end{groupplot}\end{tikzpicture}
\end{adjustbox}
\caption{Covertype (32 states, exploratory rich library): matched-information decision utility against target-label budget. Deployed excess is measured in percentage points of exact empirical-population risk.}\end{figure}

\section{Experiment 5: TENT shadow updates}
$h_0$ is fixed; a shadow copy is adapted with TENT on an unlabeled 30\% block; the decision uses labeled source images and the unlabeled target block; errors are measured on a disjoint 50\% evaluation block. CIFAR-10-C uses a ResNet-50 trained for this study (94.4\% clean accuracy). ImageNet-C uses the torchvision ResNet-50; because clean ImageNet validation data were unavailable, brightness severity 1 serves as a labeled source proxy. Three order seeds are used per cell. The procedure quantizes penultimate features into 32 $k$-means states, although the unrestricted adapted network is not constant within those states, and it \emph{declares} a within-state invariant channel that corruption violates. Consequently this entire vision panel is only a diagnostic monitor and is not a theorem-compliant deployment certificate. The plug-in witness score is $\hat c_s+\langle\hat p_t-\hat p_s,\hat w\rangle$ without concentration slack. The direct audit spends labels on the 20\% development block.

\begin{table}[H]\centering\small
\caption{Per cell: error of $h_0$ and the TENT candidate on the evaluation block (\%, mean of 3 seeds), number of harmful candidates ($>1$pp error increase), mean directional bound and plug-in score (pp), and accepted runs.}
\begin{adjustbox}{max width=\linewidth}\begin{tabular}{llr rrr r rr rrr}\toprule
Dataset & Corruption & Sev. & $h_0$ & TENT & $\Delta$ & harmful & bound & plug-in & \multicolumn{3}{c}{accepted}\\\cmidrule(lr){10-12}
& & & & & & & & & cert. & plug-in & audit\\\midrule
CIFAR-10-C & gaussian noise & 1 & 21.6 & 12.0 & -9.6 & 0/3 & +67.3 & +4.3 & 0/3 & 0/3 & 3/3 \\
CIFAR-10-C & gaussian noise & 3 & 61.6 & 22.0 & -39.5 & 0/3 & +81.8 & +2.6 & 0/3 & 0/3 & 3/3 \\
CIFAR-10-C & gaussian noise & 5 & 74.2 & 27.7 & -46.5 & 0/3 & +88.6 & +3.7 & 0/3 & 0/3 & 3/3 \\
CIFAR-10-C & defocus blur & 1 & 5.8 & 6.7 & +1.0 & 1/3 & +38.9 & +1.8 & 0/3 & 0/3 & 0/3 \\
CIFAR-10-C & defocus blur & 3 & 12.8 & 7.4 & -5.5 & 0/3 & +54.7 & +3.0 & 0/3 & 0/3 & 0/3 \\
CIFAR-10-C & defocus blur & 5 & 54.1 & 12.7 & -41.5 & 0/3 & +81.7 & +3.4 & 0/3 & 0/3 & 3/3 \\
CIFAR-10-C & fog & 1 & 5.8 & 6.9 & +1.1 & 2/3 & +39.9 & +1.9 & 0/3 & 0/3 & 0/3 \\
CIFAR-10-C & fog & 3 & 9.2 & 8.8 & -0.4 & 0/3 & +48.2 & +2.6 & 0/3 & 0/3 & 0/3 \\
CIFAR-10-C & fog & 5 & 31.9 & 16.5 & -15.3 & 0/3 & +72.2 & +4.7 & 0/3 & 0/3 & 3/3 \\
ImageNet-C & gaussian noise & 5 & 94.3 & 76.9 & -17.5 & 0/3 & +27.5 & +1.0 & 0/3 & 0/3 & 3/3 \\
ImageNet-C & defocus blur & 5 & 84.8 & 78.4 & -6.4 & 0/3 & +24.3 & +0.7 & 0/3 & 0/3 & 3/3 \\
ImageNet-C & fog & 5 & 78.0 & 45.8 & -32.2 & 0/3 & +16.9 & -1.4 & 0/3 & 3/3 & 3/3 \\
\bottomrule\end{tabular}\end{adjustbox}\end{table}
\begin{table}[H]\centering\small
\caption{Deployed error (\%) averaged over all runs, and harmful acceptances by always-adapt / direct audit.}
\begin{adjustbox}{max width=\linewidth}\begin{tabular}{l r r rrrrr r}\toprule
Dataset & runs & harmful cand. & never & always & certificate & plug-in & audit & harmful acc.\\\midrule
CIFAR-10-C & 27 & 3 & 30.8 & 13.4 & 30.8 & 30.8 & 13.8 & 3 / 0 \\
ImageNet-C & 9 & 0 & 85.7 & 67.0 & 85.7 & 75.0 & 67.0 & 0 / 0 \\
\bottomrule\end{tabular}\end{adjustbox}\end{table}
The putative certificate is vacuous for structural reasons: averaged over runs, the source term is +5.0pp and the residual-transfer term alone is +45.9pp, because the feature-state marginal moves almost entirely under corruption (mean $\widehat{\MMD}$ 0.28) and the source witness is estimated from few images per state. Moreover, quantization does not make the deep network state-constant, the channel assumption is false, and the ImageNet source is only a corruption proxy. Thus the numbers are not coverage claims and must not be used to justify deployment. Even the slack-free plug-in score is positive in most cells because the source-fitted witness does not represent the corruption-specific gains.

\section{Experiment 6: ablations}
Experiment 2 library, $d=8$, 200 repetitions; source audit 100{,}000.
\begin{table}[H]\centering\small\caption{6A: source-bound variants (directional, localized, no target labels).}
\begin{adjustbox}{max width=\linewidth}\begin{tabular}{llrrrr}\toprule Target & Source bound & $S$ & true $\mathcal{E}_s$ & bound & acc.\ (\%)\\\midrule
concentrated $\lambda=0.9$ & affine & +0.145 & +0.098 & -0.657 & 100 \\
concentrated $\lambda=0.9$ & split & +0.138 & +0.098 & -0.664 & 100 \\
concentrated $\lambda=0.9$ & min & +0.139 & +0.098 & -0.663 & 100 \\
moderate $\beta=2$ & affine & +0.147 & +0.098 & +0.032 & 0 \\
moderate $\beta=2$ & split & +0.139 & +0.098 & +0.025 & 0 \\
moderate $\beta=2$ & min & +0.141 & +0.098 & +0.026 & 0 \\
\bottomrule\end{tabular}\end{adjustbox}\end{table}
\begin{table}[H]\centering\small\caption{6B: transfer variants (minimum source bound, no target labels).}
\begin{adjustbox}{max width=\linewidth}\begin{tabular}{llrrrrr}\toprule Target & Transfer & $T$ & true $T$ & bound & acc.\ (\%) & deployed excess\\\midrule
concentrated $\lambda=0.9$ & radial & +0.910 & -0.674 & +1.245 & 0 & +0.000 \\
concentrated $\lambda=0.9$ & directional & -0.802 & -0.898 & -0.663 & 100 & -0.800 \\
concentrated $\lambda=0.9$ & joint & -0.800 & -0.898 & -0.661 & 100 & -0.800 \\
moderate $\beta=2$ & radial & +0.504 & -0.135 & +0.677 & 0 & +0.000 \\
moderate $\beta=2$ & directional & -0.114 & -0.195 & +0.026 & 0 & +0.000 \\
moderate $\beta=2$ & joint & -0.113 & -0.195 & +0.028 & 0 & +0.000 \\
\bottomrule\end{tabular}\end{adjustbox}\end{table}
\begin{table}[H]\centering\small\caption{6C: global versus localized residual and conditional penalties (audited channel, 20{,}000 target labels).}
\begin{adjustbox}{max width=\linewidth}\begin{tabular}{lllrrrr}\toprule Target & Residual & Conditional & residual term & $C$ & bound & acc.\ (\%)\\\midrule
concentrated $\lambda=0.9$ & global & global & +0.029 & +0.056 & -0.597 & 100 \\
concentrated $\lambda=0.9$ & global & localized & +0.029 & +0.038 & -0.615 & 100 \\
concentrated $\lambda=0.9$ & localized & global & +0.020 & +0.056 & -0.605 & 100 \\
concentrated $\lambda=0.9$ & localized & localized & +0.020 & +0.038 & -0.623 & 100 \\
moderate $\beta=2$ & global & global & +0.021 & +0.066 & +0.110 & 0 \\
moderate $\beta=2$ & global & localized & +0.021 & +0.026 & +0.069 & 0 \\
moderate $\beta=2$ & localized & global & +0.005 & +0.066 & +0.094 & 0 \\
moderate $\beta=2$ & localized & localized & +0.005 & +0.026 & +0.054 & 0 \\
\bottomrule\end{tabular}\end{adjustbox}\end{table}
\begin{figure}[H]\centering
\begin{adjustbox}{max width=\linewidth}
\begin{tikzpicture}\begin{groupplot}[group style={group size=2 by 2, horizontal sep=1.1cm, vertical sep=1.6cm}, width=0.42\linewidth, height=5.6cm, xmin=-0.6, xmax=1.6, xtick={0,1}, xticklabels={radial,directional}, grid=none, ymajorgrids]
\nextgroupplot[title={$\lambda=0.9$, declared channel\\$B=0$}, ymin=-1.06896, ymax=1.4163, ylabel={bound component},
 legend style={figure legend,font=\scriptsize,at={(0,1.42)},anchor=south west,legend columns=4,
 /tikz/every even column/.append style={column sep=.55em}}]
\addlegendimage{area legend, fill=cblue, draw=none}\addlegendentry{$S$}
\addlegendimage{area legend, fill=caqua, draw=none}\addlegendentry{$\langle\hat p_t-\hat p_s,\hat w\rangle$}
\addlegendimage{area legend, fill=cyellow, draw=none}\addlegendentry{concentration slack}
\addlegendimage{area legend, fill=corange, draw=none}\addlegendentry{residual transfer}
\addlegendimage{area legend, fill=cmagenta, draw=none}\addlegendentry{$C$}
\addlegendimage{area legend, fill=cviolet, draw=none}\addlegendentry{radial transfer $T$}
\addlegendimage{black, line width=1.1pt}\addlegendentry{certified bound}
\addlegendimage{black, dashed, line width=0.9pt}\addlegendentry{true target excess}
\draw[gray] (axis cs:-0.6,0) -- (axis cs:1.6,0);
\fill[cblue] (axis cs:-0.3,0) rectangle (axis cs:0.3,0.334533);
\fill[cviolet] (axis cs:-0.3,0.334533) rectangle (axis cs:0.3,1.24491);
\fill[cmagenta] (axis cs:-0.3,1.24491) rectangle (axis cs:0.3,1.24491);
\draw[black, line width=1.1pt] (axis cs:-0.38,1.24491) -- (axis cs:0.38,1.24491);
\draw[black, dashed, line width=0.9pt] (axis cs:-0.38,-0.375066) -- (axis cs:0.38,-0.375066);
\fill[cblue] (axis cs:0.7,0) rectangle (axis cs:1.3,0.138979);
\fill[caqua] (axis cs:0.7,0) rectangle (axis cs:1.3,-0.897566);
\fill[cyellow] (axis cs:0.7,0.138979) rectangle (axis cs:1.3,0.214713);
\fill[corange] (axis cs:0.7,0.214713) rectangle (axis cs:1.3,0.234746);
\fill[cmagenta] (axis cs:0.7,0.234746) rectangle (axis cs:1.3,0.234746);
\draw[black, line width=1.1pt] (axis cs:0.62,-0.66282) -- (axis cs:1.38,-0.66282);
\draw[black, dashed, line width=0.9pt] (axis cs:0.62,-0.800226) -- (axis cs:1.38,-0.800226);
\nextgroupplot[title={$\lambda=0.9$, audited channel\\$B=20{,}000$}, ymin=-1.07101, ymax=1.44397]
\draw[gray] (axis cs:-0.6,0) -- (axis cs:1.6,0);
\fill[cblue] (axis cs:-0.3,0) rectangle (axis cs:0.3,0.334533);
\fill[cviolet] (axis cs:-0.3,0.334533) rectangle (axis cs:0.3,1.24851);
\fill[cmagenta] (axis cs:-0.3,1.24851) rectangle (axis cs:0.3,1.27052);
\draw[black, line width=1.1pt] (axis cs:-0.38,1.27052) -- (axis cs:0.38,1.27052);
\draw[black, dashed, line width=0.9pt] (axis cs:-0.38,-0.375066) -- (axis cs:0.38,-0.375066);
\fill[cblue] (axis cs:0.7,0) rectangle (axis cs:1.3,0.138979);
\fill[caqua] (axis cs:0.7,0) rectangle (axis cs:1.3,-0.897566);
\fill[cyellow] (axis cs:0.7,0.138979) rectangle (axis cs:1.3,0.216177);
\fill[corange] (axis cs:0.7,0.216177) rectangle (axis cs:1.3,0.236638);
\fill[cmagenta] (axis cs:0.7,0.236638) rectangle (axis cs:1.3,0.274477);
\draw[black, line width=1.1pt] (axis cs:0.62,-0.62309) -- (axis cs:1.38,-0.62309);
\draw[black, dashed, line width=0.9pt] (axis cs:0.62,-0.800226) -- (axis cs:1.38,-0.800226);
\nextgroupplot[title={$\beta=2$, declared channel\\$B=0$}, ymin=-0.26553, ymax=0.747101]
\draw[gray] (axis cs:-0.6,0) -- (axis cs:1.6,0);
\fill[cblue] (axis cs:-0.3,0) rectangle (axis cs:0.3,0.173342);
\fill[cviolet] (axis cs:-0.3,0.173342) rectangle (axis cs:0.3,0.677265);
\fill[cmagenta] (axis cs:-0.3,0.677265) rectangle (axis cs:0.3,0.677265);
\draw[black, line width=1.1pt] (axis cs:-0.38,0.677265) -- (axis cs:0.38,0.677265);
\draw[black, dashed, line width=0.9pt] (axis cs:-0.38,0.00409092) -- (axis cs:0.38,0.00409092);
\fill[cblue] (axis cs:0.7,0) rectangle (axis cs:1.3,0.14055);
\fill[caqua] (axis cs:0.7,0) rectangle (axis cs:1.3,-0.195694);
\fill[cyellow] (axis cs:0.7,0.14055) rectangle (axis cs:1.3,0.216781);
\fill[corange] (axis cs:0.7,0.216781) rectangle (axis cs:1.3,0.222066);
\fill[cmagenta] (axis cs:0.7,0.222066) rectangle (axis cs:1.3,0.222066);
\draw[black, line width=1.1pt] (axis cs:0.62,0.0263722) -- (axis cs:1.38,0.0263722);
\draw[black, dashed, line width=0.9pt] (axis cs:0.62,-0.0977334) -- (axis cs:1.38,-0.0977334);
\nextgroupplot[title={$\beta=2$, audited channel\\$B=20{,}000$}, ymin=-0.267173, ymax=0.76929]
\draw[gray] (axis cs:-0.6,0) -- (axis cs:1.6,0);
\fill[cblue] (axis cs:-0.3,0) rectangle (axis cs:0.3,0.173342);
\fill[cviolet] (axis cs:-0.3,0.173342) rectangle (axis cs:0.3,0.680367);
\fill[cmagenta] (axis cs:-0.3,0.680367) rectangle (axis cs:0.3,0.69781);
\draw[black, line width=1.1pt] (axis cs:-0.38,0.69781) -- (axis cs:0.38,0.69781);
\draw[black, dashed, line width=0.9pt] (axis cs:-0.38,0.00409092) -- (axis cs:0.38,0.00409092);
\fill[cblue] (axis cs:0.7,0) rectangle (axis cs:1.3,0.14055);
\fill[caqua] (axis cs:0.7,0) rectangle (axis cs:1.3,-0.195692);
\fill[cyellow] (axis cs:0.7,0.14055) rectangle (axis cs:1.3,0.218257);
\fill[corange] (axis cs:0.7,0.218257) rectangle (axis cs:1.3,0.223676);
\fill[cmagenta] (axis cs:0.7,0.223676) rectangle (axis cs:1.3,0.249252);
\draw[black, line width=1.1pt] (axis cs:0.62,0.0535594) -- (axis cs:1.38,0.0535594);
\draw[black, dashed, line width=0.9pt] (axis cs:0.62,-0.0977328) -- (axis cs:1.38,-0.0977328);
\end{groupplot}\end{tikzpicture}
\end{adjustbox}
\caption{Experiment 6: component-wise decomposition of the selected candidate's bound (means over 200 repetitions; minimum source bound, localized penalties). Positive components stack upward, negative downward.}\end{figure}

%% file: references.bib
@inproceedings{wu2022split,
  title={Split-kl and {PAC}-{B}ayes-split-kl Inequalities for Ternary Random Variables},
  author={Wu, Yi-Shan and Seldin, Yevgeny},
  booktitle={Advances in Neural Information Processing Systems},
  volume={35},
  year={2022},
  eprint={2206.00706},
  archivePrefix={arXiv}
}

@article{alquier2024user,
  title={User-friendly introduction to {PAC}-{B}ayes bounds},
  author={Alquier, Pierre},
  journal={Foundations and Trends in Machine Learning},
  volume={17},
  number={2},
  pages={174--303},
  year={2024},
  eprint={2110.11216},
  archivePrefix={arXiv}
}

@inproceedings{amit2022ipm,
  title={Integral Probability Metrics {PAC}-{B}ayes Bounds},
  author={Amit, Ron and Epstein, Baruch and Moran, Shay and Meir, Ron},
  booktitle={Advances in Neural Information Processing Systems},
  volume={35},
  year={2022}
}

@inproceedings{biggs2023tighter,
  title={Tighter {PAC}-{B}ayes Generalisation Bounds by Leveraging Example Difficulty},
  author={Biggs, Felix and Guedj, Benjamin},
  booktitle={Proceedings of the 26th International Conference on Artificial Intelligence and Statistics},
  volume={206},
  pages={8165--8182},
  year={2023}
}

@inproceedings{petrik2016safe,
  title={Safe Policy Improvement by Minimizing Robust Baseline Regret},
  author={Petrik, Marek and Ghavamzadeh, Mohammad and Chow, Yinlam},
  booktitle={Advances in Neural Information Processing Systems},
  volume={29},
  year={2016}
}

@article{breitholtz2022practicality,
  title={Practicality of generalization guarantees for unsupervised domain adaptation with neural networks},
  author={Breitholtz, Adam and Johansson, Fredrik D.},
  journal={Transactions on Machine Learning Research},
  year={2022},
  eprint={2303.08720},
  archivePrefix={arXiv}
}

@inproceedings{germain2013pac,
  title={A {PAC}-{B}ayesian Approach for Domain Adaptation with Specialization to Linear Classifiers},
  author={Germain, Pascal and Habrard, Amaury and Laviolette, Fran{\c{c}}ois and Morvant, Emilie},
  booktitle={Proceedings of the 30th International Conference on Machine Learning},
  volume={28},
  pages={738--746},
  year={2013}
}

@inproceedings{germain2016domain,
  title={A New {PAC}-{B}ayesian Perspective on Domain Adaptation},
  author={Germain, Pascal and Habrard, Amaury and Laviolette, Fran{\c{c}}ois and Morvant, Emilie},
  booktitle={Proceedings of the 33rd International Conference on Machine Learning},
  volume={48},
  pages={859--868},
  year={2016}
}

@article{germain2015majority,
  title={Risk Bounds for the Majority Vote: From a {PAC}-{B}ayesian Analysis to a Learning Algorithm},
  author={Germain, Pascal and Lacasse, Alexandre and Laviolette, Fran{\c{c}}ois and Marchand, Mario and Roy, Jean-Francis},
  journal={Journal of Machine Learning Research},
  volume={16},
  number={26},
  pages={787--860},
  year={2015}
}

@misc{maurer2004note,
  title={A Note on the {PAC} {B}ayesian Theorem},
  author={Maurer, Andreas},
  year={2004},
  eprint={cs/0411099},
  archivePrefix={arXiv}
}

@inproceedings{niu2022eata,
  title={Efficient Test-Time Model Adaptation without Forgetting},
  author={Niu, Shuaicheng and Wu, Jiaxiang and Zhang, Yifan and Chen, Yaofo and Zheng, Shijian and Zhao, Peilin and Tan, Mingkui},
  booktitle={Proceedings of the 39th International Conference on Machine Learning},
  volume={162},
  pages={16888--16905},
  year={2022}
}

@inproceedings{niu2023sar,
  title={Towards Stable Test-Time Adaptation in Dynamic Wild World},
  author={Niu, Shuaicheng and Wu, Jiaxiang and Zhang, Yifan and Wen, Zhiquan and Chen, Yaofo and Zhao, Peilin and Tan, Mingkui},
  booktitle={International Conference on Learning Representations},
  year={2023}
}

@inproceedings{staib2019mmd,
  title={Distributionally Robust Optimization and Generalization in Kernel Methods},
  author={Staib, Matthew and Jegelka, Stefanie},
  booktitle={Advances in Neural Information Processing Systems},
  volume={32},
  year={2019}
}

@inproceedings{wang2021tent,
  title={Tent: Fully Test-Time Adaptation by Entropy Minimization},
  author={Wang, Dequan and Shelhamer, Evan and Liu, Shaoteng and Olshausen, Bruno and Darrell, Trevor},
  booktitle={International Conference on Learning Representations},
  year={2021}
}

@misc{yu2023benchmark,
  title={Benchmarking Test-Time Adaptation against Distribution Shifts in Image Classification},
  author={Yu, Yongcan and Sheng, Lijun and He, Ran and Liang, Jian},
  year={2023},
  eprint={2307.03133},
  archivePrefix={arXiv}
}

@inproceedings{zhang2022memo,
  title={{MEMO}: Test Time Robustness via Adaptation and Augmentation},
  author={Zhang, Marvin and Levine, Sergey and Finn, Chelsea},
  booktitle={Advances in Neural Information Processing Systems},
  volume={35},
  year={2022}
}

@inproceedings{zhu2021kernel,
  title={Kernel Distributionally Robust Optimization: Generalized Duality Theorem and Stochastic Approximation},
  author={Zhu, Jia-Jie and Jitkrittum, Wittawat and Diehl, Moritz and Sch{\"o}lkopf, Bernhard},
  booktitle={Proceedings of the 24th International Conference on Artificial Intelligence and Statistics},
  volume={130},
  pages={280--288},
  year={2021}
}

@misc{blackard1998covertype,
  author={Blackard, Jock A.},
  title={Covertype},
  year={1998},
  howpublished={UCI Machine Learning Repository},
  doi={10.24432/C50K5N}
}

@inproceedings{hendrycks2019benchmarking,
  author={Hendrycks, Dan and Dietterich, Thomas},
  title={Benchmarking Neural Network Robustness to Common Corruptions and Perturbations},
  booktitle={International Conference on Learning Representations},
  year={2019}
}

@inproceedings{schirmer2025monitoring,
  author={Schirmer, Mona and Jazbec, Metod and Naesseth, Christian A. and Nalisnick, Eric},
  title={Monitoring Risks in Test-Time Adaptation},
  booktitle={Advances in Neural Information Processing Systems},
  volume={38},
  year={2025}
}
